\documentclass[11pt]{article}
\pdfoutput=1
\usepackage{enumerate}
\usepackage{pdfsync}
\usepackage[OT1]{fontenc}

\usepackage[usenames]{color}
\usepackage[colorlinks,
            linkcolor=red,
            anchorcolor=blue,
            citecolor=blue,
            linktocpage=true,
            linkcolor=blue            
            ]{hyperref}
\usepackage{fullpage}
\usepackage{hyperref}
\usepackage[protrusion=true,expansion=true]{microtype}
\usepackage{pbox}
\usepackage{setspace}
\usepackage{tabularx}
\usepackage{float}
\usepackage{wrapfig,lipsum}
\usepackage{enumitem}
\usepackage{natbib}
\usepackage{multirow} 

\usepackage[ruled,vlined,linesnumbered]{algorithm2e}

\let\oldnl\nl
\newcommand{\nonl}{\renewcommand{\nl}{\let\nl\oldnl}}

\makeatletter
\newcommand*{\rom}[1]{\expandafter\@slowromancap\romannumeral #1@}
\makeatother

\usepackage{dsfont}
\usepackage{booktabs}
\usepackage{macro}

\allowdisplaybreaks

\def \dd{\text{d}}
\def \hmu{\hat\mu}
\def \stageOne{\emph{Stage I}}
\def \stageTwo{\emph{Stage II}}
\def \stageThree{\emph{Stage III}}
\def \stageFour{\emph{Stage IV}}
\def \failflag{\cF_\text{fail}}

\begin{document}
\title{\huge  Double Explore-then-Commit: Asymptotic Optimality and Beyond}
\author
{
	Tianyuan Jin\thanks{School of Computing, National University of Singapore, Singapore; e-mail: {\tt Tianyuan1044@gmail.com}} 
	~~~and~~~
	Pan Xu\thanks{Department of Computer Science, University of California, Los Angeles, Los Angeles, CA 90095; e-mail: {\tt panxu@cs.ucla.edu}} 
	~~~and~~~
	Xiaokui Xiao\thanks{School of Computing, National University of Singapore, Singapore; e-mail: {\tt xkxiao@nus.edu.sg}} 
	~~~and~~~
	Quanquan Gu\thanks{Department of Computer Science, University of California, Los Angeles, Los Angeles, CA 90095; e-mail: {\tt qgu@cs.ucla.edu}}
}
\date{}
\maketitle

\begin{abstract}%
We study the multi-armed bandit problem with subgaussian rewards. The explore-then-commit (ETC) strategy, which consists of an exploration phase followed by an exploitation phase, is one of the most widely used algorithms in a variety of online decision applications. Nevertheless, it has been shown in \cite{garivier2016explore} that ETC is suboptimal in the asymptotic sense as the horizon grows, and thus, is worse than fully sequential strategies such as Upper Confidence Bound (UCB). In this paper, we show that a variant of ETC algorithm can actually achieve the asymptotic optimality for multi-armed bandit problems as UCB-type algorithms do and extend it to the batched bandit setting. Specifically, we propose a double explore-then-commit (DETC) algorithm that has two exploration and exploitation phases and prove that DETC achieves the asymptotically optimal regret bound. To our knowledge, DETC is the first non-fully-sequential algorithm that achieves such asymptotic optimality. In addition, we extend DETC to batched bandit problems, where (i) the exploration process is split into a small number of batches and (ii) the round complexity\footnote{Round complexity is defined as the total number of times an algorithm needs to update its learning policy. For instance, an UCB algorithm on a bandit problem with time horizon $T$ will have $O(T)$ round complexity because it needs to update its estimation for arms based on the reward collected at each time step.} is of central interest. We prove that a batched version of DETC can achieve the asymptotic optimality with only a constant round complexity. This is the first batched bandit algorithm that can attain the optimal asymptotic regret bound and optimal round complexity simultaneously.
\end{abstract}



\section{Introduction}
We study the multi-armed bandit problem, where an agent is asked to choose a bandit arm $A_t$ from a set of arms $\{1, 2, \ldots, K\}$ at every time step $t$. Then it observes a reward $r_t$ associated with arm $A_t$ following a 1-subgaussian distribution with an unknown mean value $\mu_{A_t}$. For an arbitrary horizon length $T$, the performance of any strategy for the bandit problem is measured by the {\it expected cumulative regret}, which is defined as:
\begin{equation}\label{eq:def_regret_2arm}
    R_{\mu}(T)= T \cdot \max\limits_{i\in \{1,2,\cdots,K\}} \mu_i - \mathbb{E}_{\mu} \bigg[ \sum_{t=1}^T r_t\bigg],
\end{equation}
where the subscript $\mu$ denotes the bandit instance consisting of the $K$ arms $\{\mu_1,\ldots,\mu_K\}$.


Assume without loss of generality that arm $1$ has the highest mean, i.e., $\mu_1=\max\{\mu_1,\ldots,\mu_K\}$.  \citet{lai1985asymptotically,katehakis1995sequential}  show that when each arm's reward distribution is Gaussian,
the expected regret of any strategy is at least $\sum_{i:\Delta_i>0} \frac{2\log T}{\Delta_i}$ when $T$ approaches infinity, where $\Delta_i=|\mu_1-\mu_i|$ denotes the difference between the mean rewards of arm $1$ and $i$. That is,
\begin{equation}\label{def:asy_opt_unknown}
   \liminf\limits_{T\rightarrow \infty} \frac{R_{\mu}(T)}{\log T}\geq \sum_{i:\Delta_i>0}\frac{2}{\Delta_i}.
\end{equation}
When $\Delta_i$ ($i = 1, 2, \ldots, K$) are known to the decision maker in advance, \cite{garivier2016explore} show that the asymptotic lower bound turns to
\begin{equation}\label{def:asy_opt_known}
   {\liminf\limits_{T\rightarrow \infty}} \frac{R_{\mu}(T)}{\log T}\geq \sum_{i:\Delta_i>0}\frac{1}{2\Delta_i}.
\end{equation}
We refer to ${\lim}_{T \rightarrow\infty} R_{\mu}(T)/\log T$ as the {\it asymptotic regret rate}, and we say that an algorithm is {\it asymptotically optimal} if it achieves the regret lower bound in \eqref{def:asy_opt_unknown} (when $\Delta_i$ are unknown) or \eqref{def:asy_opt_known} (when $\Delta_i$ are known).

There exist a number of multi-armed bandit algorithms (e.g., UCB~\citep{katehakis1995sequential,garivier2011kl}, Thompson Sampling~\citep{agrawal2017near,korda2013thompson}, Bayes UCB~\citep{kaufmann2018bayesian}) that are asymptotically optimal. All of these algorithms, however, are {\it sequential}, in the sense that they need to observe the outcome of each arm pull before deciding which arm should be pulled next.  
Such sequential algorithms are unsuitable for applications where each arm pull take a substantial amount of time. For example, in clinical trials, each treatment involving a human participate can be regarded as an arm pull, and the outcome of the treatment can only be observed after a defined time period. It is thus infeasible to conduct all treatments in a sequential manner due to the prohibitive total time cost. 
In such applications, a more preferable strategy is to pull arms simultaneously to reduce the waiting time for outcomes. Motivated by this, existing work \citep{perchet2016batched,bertsimas2007learning,chick2009economic,agarwal2017learning,NIPS2019_8341,esfandiari2019batched} has studied the {\it batched bandit} problem, which requires arms to be pulled in {\em rounds}. In each round, we are allowed to pull multiple arms at the same time, but can only only observe the outcomes at the end of the round. The problem asks for a strategy that minimizes not only the expected cumulative regret after $T$ arm pulls, but also the {\it number of rounds}.


In contrast to fully-sequential bandit algorithms, strategies with distinct exploration and exploitation stages are often more efficient and achieve lower round complexities in batched bandits, where outcomes are only needed at the stage switching time.  
The most natural approach for separating the exploration and exploitation stages is to first pull each arm for a fixed number of times (the exploration stage), and then pull the arm with the larger average reward repeatedly based on the result in the previous stage (the exploitation stage). The length of the exploration stage can be a data-dependent stopping time. Such strategies with distinct exploration and exploitation stages fall into the class of approaches named explore-then-commit (ETC) \citep{perchet2016batched,garivier2016explore}, which are simple and widely implemented in various online applications, such as clinical trials, crowdsourcing and marketing~\citep{perchet2016batched,garivier2016explore,NIPS2019_8341}.  
Regarding the regret analysis, \citet{garivier2016optimal} suggested that carefully-tuned variants of such two-stage strategies might be near-optimal. Yet  \citet{garivier2016explore} later proved that a class of two-stage ETC strategies are actually suboptimal in the sense that they cannot achieve the asymptotically optimal lower bounds in \eqref{def:asy_opt_unknown} or \eqref{def:asy_opt_known}. Existing batched bandit algorithms~\citep{perchet2016batched,NIPS2019_8341,esfandiari2019batched} are based on two-stage ETC, hence is suboptimal in the asymptotic sense.  To this end, a natural and open question is:
\begin{center}
    \textit{Can non-fully-sequential strategies such as more stages ETC strategies \\achieve the optimal regret?}
\end{center}

In this paper, we answer the above question affirmatively by proposing a double explore-then-commit (DETC) algorithm that consists of two exploration and two exploitation stages, which directly improves the ETC algorithm proposed in \cite{garivier2016explore}. 
Take the two-armed bandit problem as an example, the key idea of DETC is illustrated as follows: based on the result of the first exploration stage, the algorithm will commit to the arm with the largest average reward and pull it for a long time in the exploitation stage. After the first exploitation stage, the algorithm will have a confident estimate of the chosen arm. However, since the unchosen arm is never pulled after the first exploration stage, the algorithm is still not sure whether the unchosen arm is underestimated. Therefore, a second exploration stage for the algorithm to pull the unchosen arms is necessary. After this stage, the algorithm will have sufficiently accurate estimate for all arms and just needs another exploitation stage to commit to the arm with the largest average reward. 
In contrast to the above double explore-then-commit algorithm, existing ETC algorithms may have inaccurate estimates for both the optimal arm and the suboptimal arms and hence suffers a suboptimal regret.


\vspace{1ex}
\subsection{Our Contributions}
    We first study the two-armed bandit problem. In this case, we simplify the notation by denoting $\Delta=\Delta_2=|\mu_1-\mu_2|$ as the gap. When the gap $\Delta$ is a known parameter to the algorithm, we prove that  DETC  achieves the asymptotically optimal regret rate  $1/(2\Delta)$, the instance-dependent optimal regret $O(\log(T\Delta^2)/\Delta)$ and the minimax optimal regret $O(\sqrt{T})$ for two-armed bandits. 
This result significantly improves the $4/\Delta$ asymptotic regret  rate of   ETC with fixed length and the $1/\Delta$ asymptotic regret  rate of SPRT-ETC with data-dependent stopping time for the exploration stage proposed in \cite{garivier2016explore}. 

When $\Delta$ is unknown, we prove that the DETC strategy achieves the  asymptotically optimal regret rate {$2/\Delta$}, the instance-dependent regret $O(\log(T\Delta^2)/\Delta)$ and the minimax optimal regret $O(\sqrt{T})$. This again improves the $4/\Delta$ asymptotic regret  rate of the BAI-ETC algorithm proposed in \cite{garivier2016explore}. In both the known gap and the unknown gap settings, this is the first time that the regrets of ETC algorithms have been proved to match the asymptotic lower bounds and therefore are asymptotically optimal. In contrast, \citet{garivier2016explore} proved that the $1/\Delta$ asymptotic regret  rate for the known gap case and the $4/\Delta$ asymptotic regret  rate for the unknown gap case are not improvable in `single' explore-then-commit algorithms, which justifies the essence of the double exploration technique in  DETC in order to achieve the asymptotic regret. 

We also propose a variant of DETC that is simultaneously instance-dependent/minimax optimal and asymptotically optimal for two-armed bandit problems. Our analysis and algorithmic framework also suggests an effective way of combining the asymptotically optimal DETC algorithm with any other minimax optimal algorithms to achieve the  instance-dependent/minimax  and asymptotic optimality simultaneously. 
We further extend our DETC algorithm to $K$-armed bandit problems and prove that DETC achieves the asymptotically optimal regret rate $\sum_{i:\Delta_i>0}2/\Delta_i$ for $K$-armed bandits \citep{lai1985asymptotically}, where $\Delta_i$ is the gap between the best arm and the $i$-th arm, $i\in[K]$. 

To demonstrate the advantages and potential applications of our double explore-then-commit strategy, we also study the batched bandits problem \citep{perchet2016batched} where the round complexity is of central interest. We prove that a simple variant of the proposed DETC algorithm can achieve $O(1)$ round complexity while maintaining the asymptotically optimal regret for two-armed bandits. This is a significant improvement of the round complexity of fully sequential strategies such as UCB and UCB2 \citep{lai1985asymptotically,auer2002finite,garivier2011kl}, which usually requires $O(T)$ or $O(\log T)$ rounds. This is the first batched bandit algorithm that achieves the asymptotic optimality in regret and the optimal round complexity. Our result also suggests that it is not necessary to use the outcome at each time step as in fully sequential algorithms such as UCB to achieve the asymptotic optimality.


\vspace{1ex}
\noindent\textbf{Notation} We denote $\log^{+} (x)=\max \{0, \log x\}$. We use notations $\lfloor x\rfloor$ (or $\lceil x\rceil$) to denote the largest integer that is no larger (or no smaller) than $x$. We use $O(T)$ to hide constants that are independent of $T$. A random variable $X$ is said to follow 1-subgaussian distribution, if it holds that $\EE_{X}[\exp(\lambda X-\lambda \EE_{X}[X])]\leq \exp(\lambda^2/2)$ for all $\lambda \in \RR$.


\section{Related Work}
For regret minimization in stochastic bandit problems, \cite{lai1985asymptotically} proved the first asymptotically lower bound that any strategy must have at least  $C(\mu)\log(T)(1-o(1))$ regret when the horizon $T$ approaches infinity, where $C(\mu)$ is a constant. Later, strategies such as UCB~\citep{lai1985asymptotically,auer2002finite,garivier2011kl}, Thompson Sampling~\citep{korda2013thompson,agrawal2017near} and Bayes UCB~\citep{kaufmann2018bayesian} are all shown to be asymptotically optimal in the unknown gap setting. For the known gap setting, \cite{garivier2016explore} developed the $\Delta$-UCB algorithm that matches the lower bound. To our knowledge, all previous asymptotically optimal algorithms are fully sequential.  Despite the asymptotic optimality,  for a fixed time horizon $T$, the problem-independent lower bound~\citep{auer2002nonstochastic} states that any strategy has at least a regret in the order of $\Omega(\sqrt{KT})$, which is called the \emph{minimax optimal} regret.  MOSS~\citep{audibert2009minimax} is the first method proved to be minimax optimal. Subsequently, two UCB-based methods,  AdaUCB~\citep{lattimore2018refining} and KL-UCB$^{++}$~\citep{menard2017minimax}, are also shown to achieve minimax optimality. 


There is less work yet focusing on the batched bandit setting with limited rounds. UCB2 \citep{auer2002finite}, which needs implicitly $O(\log T)$ rounds of queries, is a variant of UCB that takes $O(T)$ rounds of queries. 
\citet{cesa2013online} studied the batched bandit problem under the notion of switching cost and showed that $\log \log T$ rounds are sufficient to achieve the minimax optimal regret~\citep{audibert2009minimax}. \citet{perchet2016batched} studied the two-armed batched bandit problem with limited rounds. They developed polices that is minimax optimal and proved that their round cost is near optimal.  \citet{NIPS2019_8341} used similar polices for $K$-armed batched bandits and proved that their batch complexity and regret are both near optimal, which is recently further improved by \cite{esfandiari2019batched}. Besides, \cite{NIPS2019_8341,esfandiari2019batched,perchet2016batched} also provide the instance dependent regret bound under the limited rounds setting. In the asymptotic sense, the regret bound is $O(K\log T)$. However, the hidden constant in $O(K\log T)$ makes it suboptimal in terms of the asymptotic regret and the round cost of these works is $\Theta(\log T)$. In addition, the batched bandit problem is also studied in the linear bandit setting \citep{esfandiari2019batched,han2020sequential,ruan2020linear}, best arm identification~\citep{agarwal2017learning,tianyuan2019efficient} and in theoretical computer science under
the name of {\it parallel algorithms} \citep{valiant1975parallelism,tao2019collaborative,alon1988sorting,feige1994computing,bollobas1983parallel,ajtai1986deterministic,braverman2016parallel,duchi2018minimax}, to mention a few.

\section{Double Explore-then-Commit Strategies}
The vanilla ETC strategy \citep{perchet2016batched,garivier2016explore} consists of two stages: in stage one (the exploration stage), the agent pulls all arms for the same number of times, which can be a fixed integer or a data-dependent stopping time, leading to the FB-ETC and SPRT-ETC (or BAI-ETC) algorithms in \citep{garivier2016explore}; in stage two (the exploitation stage), the agent pulls the arm that achieves the best average reward according to the outcome of stage one. As we mentioned in the introduction, none of these algorithms can achieve the asymptotic optimality in \eqref{def:asy_opt_unknown} or \eqref{def:asy_opt_known}. To tackle this problem, we propose a double explore-then-commit strategy for two-armed bandits that improves ETC to be asymptotically optimal while still keeping non-fully-sequential. 


\subsection{Warm-Up: The Known Gap Setting}\label{sec:detc_known}
We first consider the case where the gap $\Delta=\mu_1-\mu_2$ is known to the decision maker (recall that we assume w.l.o.g. that arm $1$ is the optimal arm). We propose a double explore-then-commit (DETC) algorithm, which consists of four stages. The details are displayed in Algorithm \ref{alg:double-exploration-known-gap}.

\begin{algorithm}[t]
\caption{Double Explore-then-Commit  (DETC) in the Known Gap Setting}
\label{alg:double-exploration-known-gap}
\KwIn {$T$, $\epsilon_T$ and $\Delta$.}
\textbf{Initialization:} Pull arms $A_1=1$ and $A_2=2$, $t\leftarrow 2$, $T_1= \lceil 2\log (T\Delta^2)/(\epsilon_T^2\cdot \Delta^2)\rceil$, $\tau_1=4\lceil\log(T_1\Delta^2)/\Delta^2\rceil$\; 
\nonl \hrulefill\\ 
\nonl \textit{Stage I: Explore all arms uniformly}\\  
\While{$t\leq 2\tau_1$} 
    { 
    Pull arms  $A_{t+1}=1$ and $A_{t+2}=2$, $t\leftarrow t+2$\;} 
\nonl \hrulefill\\ 
\nonl \textit{Stage II:  Commit to the arm with the largest average reward}\\ 
${1'}\leftarrow\arg \max_{k\in\{1,2\}} \hat{\mu}_k(t)$\;
\While{$T_{1'}(t)\leq T_1$}
    { 
    Pull arm  $A_{t+1}={1'}$, $t\leftarrow t+1$\;}
\nonl \hrulefill\\ 
\nonl \textit{Stage III: Explore the unchosen arm in Stage II}\\
$\mu'\leftarrow \hat{\mu}_{1'}(t) $, $t_2\leftarrow 0$, $2'\leftarrow \{1,2\}\setminus {1'}$, 
$\theta_{2',0}\leftarrow 0$\;\label{algline:detc_known_stage3}
\While{${2(1-\epsilon_T)t_2\Delta}\mid \mu'-\theta_{2',t_2} \mid <\log(T\Delta^2)$}
    {
    Pull arm $A_{t+1}=2'$ and observe reward $r_{t+1}$\; $\theta_{2',t_2+1}=(t_2\theta_{2',t_2}+r_{t+1})/(t_2+1)$, $t\leftarrow t+1$, $t_2\leftarrow t_2+1$\;}
\nonl \hrulefill \\
\nonl \textit{Stage IV: Commit to the arm with the largest average reward }\\
$a\leftarrow 1'\ind\{\hat{\mu}_{1'}(t)\geq\theta_{2',t_2}\}+2'\ind\{\hat{\mu}_{1'}(t)<\theta_{2',t_2}\}$\; 
\While{$t\leq T$}
    { 
    Pull arm $a$, $t\leftarrow t+1$\;}
\end{algorithm}

At the initialization step, we  pull both arms once, after which we set the current time step $t=2$. In $\stageOne$, DETC pulls both arms for $\tau_1=4\lceil\log(T_1\Delta^2)/\Delta^2\rceil$ times respectively, where both $\tau_1$ and $T_1$ are predefined parameters. At time step $t$, we define $T_{k}(t)$ to be the total number of times that arm $k$ ($k=1,2$) has been pulled so far, i.e., {$T_k(t)=\sum_{i=1}^t\ind_{\{A_i=k\}}$, where $A_i$ is the arm pulled at time step $i$.  Then we define the average reward of arm $k$ at time step $t$ as $\hmu_k(t):=\sum_{i=1}^t\ind_{\{A_i=k\}}r_i/T_k(t)$, where $r_i$ is the reward received by the algorithm at time $i$. }

In $\stageTwo$, DETC repeatedly pulls the arm with the largest average reward at the end of $\stageOne$, denoted by arm  ${1'}=\arg \max_{k=1,2}\hat{\mu}_{k,\tau_1}$, where $\hat{\mu}_{k,\tau_1}$ is the average reward of arm $k$ after its $\tau_1$-$th$ pull. Note that before $\stageTwo$, arm $1'$ has been pulled for {$\tau_1$ times}. We will terminate $\stageTwo$ after the total number of pulls of arm $1'$ reaches $T_1$. It is worth noting that $\stageOne$ and $\stageTwo$ are similar to existing ETC algorithms \citep{garivier2016explore}, where these two stages are referred to as the $\emph{Explore}$ (explore different arms) and the $\emph{Commit}$ (commit to one single arm) stages respectively. 

The key difference here is that instead of pulling arm $1'$ till the end of the horizon (time step $T$), our Algorithm \ref{alg:double-exploration-known-gap} sets a check point $T_1<T$. After arm $1'$ has been pulled for $T_1$ times, we stop and check the average reward of the arm that is not chosen in $\stageTwo$, denoted by arm $2'$. The motivation for this halting follows from a natural question: 
\emph{What if we have committed to the wrong arm?} Even though arm $2'$ is not chosen based on the outcome of $\stageOne$, it can still be optimal due to random sampling errors.  To avoid such a case, we pull arm $2'$ for more steps such that the average rewards of both arms can be  distinguished from each other. Specifically, in $\stageThree$ of Algorithm \ref{alg:double-exploration-known-gap}, arm $2'$ is repeatedly pulled until
\begin{align}\label{eq:detc_stage3_stopping_rule}
    2(1-\epsilon_T)t_2\Delta|\mu'-\theta_{2',t_2}|\geq\log(T\Delta^2),
\end{align}
where $\epsilon_T>0$ is a parameter, $t_2$ is the total number of pulls in Stage \emph{III}, $\theta_{2',t_2}$ is the average reward of arm $2'$ in $\stageThree$ and $\mu'$ is the average reward of arm $1'$ recorded at the end of $\stageTwo$. Note that $\mu'=\hmu_{1'}(t)$ throughout $\stageThree$ since arm $1'$ is not pulled in this stage. 

As is discussed in the above paragraph, at the end of $\stageTwo$, the average reward $\mu'$ for arm $1'$ already concentrates on its expected reward. Therefore, in $\stageThree$ of DETC, the sampling error only comes from pulling arm $2'$. Hence, our DETC algorithm offsets the drawback ETC algorithms where  the sampling error comes from both arms. In the remainder of the algorithm ($\stageFour$), we just again commit to the arm with the largest empirical reward from at the end of $\stageThree$.

Now, we present the regret bound of Algorithm \ref{alg:double-exploration-known-gap}. Note that if $T\Delta^2<1$, the worst case regret  is trivially bounded by $T\Delta<\sqrt{T}$ and the asymptotic regret rate is meaningless since $\Delta\rightarrow0$ when $T\rightarrow\infty$. Hence, in the following theorem, we assume $T\Delta^2\geq 1$.
\begin{theorem}\label{theorem:knowndelta_2arm}
If $\epsilon_{T}$ is chosen such that $T_1\Delta^2\geq 1$, the regret of Algorithm~\ref{alg:double-exploration-known-gap} is upper bounded as 
\begin{equation}\label{eq:regret_known_gap}
  R_{\mu}(T)\leq 2\Delta+ \frac{8}{\Delta}+\frac{4\log(T_1\Delta^2)}{\Delta}+\frac{\log(T\Delta^2)}{2(1-\epsilon_T)^2\Delta}+\frac{2\sqrt{\log(T\Delta^2)}+2}{(1-\epsilon_T)^2\Delta}.
\end{equation}
In particular, let $\epsilon_T=\min\{\sqrt{{\log(T\Delta^2)}/({\Delta^2\log^2T})}, 1/2\}$, then $\limsup_{T\rightarrow \infty}  R_{\mu}(T)/\log T\leq 1/(2\Delta)$, and $R_{\mu}(T)=O(\Delta+\log(T\Delta^2)/{\Delta})=O(\Delta+\sqrt{T})$.
\end{theorem}
The proof of Theorem \ref{theorem:knowndelta_2arm} can be found in Section \ref{sec:proof_of_thm_knowdelta_2arm}. This theorem states that Algorithm \ref{alg:double-exploration-known-gap}  achieves the asymptotically optimal regret  rate $1/(2\Delta)$, instance-dependent optimal regret $O(\Delta+1/\Delta\log(T\Delta^2))$ and minimax regret $O(\Delta+\sqrt{T})$, when parameter $\epsilon_T$ is properly chosen. In comparison, the ETC algorithm in \cite{garivier2016explore} can only achieve $1/\Delta$ asymptotic regret  rate under the same setting, which is suboptimal for multi-armed bandit problems \citep{lai1985asymptotically} when gap $\Delta$ is known to the decision maker. It is important to note that,  \citet{garivier2016explore} also proved a lower bound for asymptotic optimality of ETC and showed that the $1/\Delta$ asymptotic regret  rate of `single' explore-then-commit algorithms cannot be improved. Therefore, the double exploration techniques in our DETC is indeed essential for breaking the $1/\Delta$ barrier in the asymptotic regret  rate. 


The asymptotic optimality is also achieved by the $\Delta$-UCB algorithm in ~\cite{garivier2016explore}, which is a fully sequential strategy. 
In stark contrast, DETC shows that non-fully-sequential  algorithms can also achieve the asymptotically optimal regret for multi-armed bandit problems. Compared with $\Delta$-UCB, DETC has distinct stages of exploration and exploitation which makes the implementation simple and more practical. A more important and unique feature of DETC is its lower round complexity for batched bandit problems, which will be thoroughly discussed in Section \ref{sec:batch_known}.


\subsection{Double Explore-then-Commit in the Unknown Gap Setting}\label{sec:detc_unknown}
In real world applications, the gap $\Delta$ is often unknown. Thus, it is favorable to  design an algorithm without the knowledge of $\Delta$. However, this imposes issues with Algorithm \ref{alg:double-exploration-known-gap}, since the stopping rules of the two exploration stages ($\stageOne$ and $\stageThree$) are unknown. To address this challenge, we propose a DETC algorithm where the gap $\Delta$ is unknown to the decision maker, which is displayed in Algorithm \ref{alg:double-exploration-unknown-gap}. 
\begin{algorithm}[t]
   \caption{Double Explore-then-Commit (DETC) in the Unknown Gap Setting}
   \label{alg:double-exploration-unknown-gap}
  \KwIn{$T, T_1$}
 \textbf{Initialization:} Pull arms $A_1=1$, $A_2=2$, $t\leftarrow 2$; \\
\nonl \hrulefill \\
\nonl  \textit{Stage I: Explore all arms uniformly} \\  
  \While{$\mid \hmu_1(t)-\hat{\mu}_2(t)\mid<\sqrt{16/t\log^{+}(T_1/t)}$} 
     {
       Pull arms  $A_{t+1}=1$ and $A_{t+2}=2$, $t\leftarrow t+2$\;
     }
\nonl \hrulefill \\ 
\nonl \textit{Stage II: Commit to the arm with the largest average reward} \\ 
  $1'\leftarrow \arg \max_{i} \hat{\mu}_i(t)$\;
  \While{$T_{1'}(t)\leq {T_1}$}
 { Pull arm  $A_{t+1}=1'$, $t\leftarrow t+1$\;}
\nonl \hrulefill\\ 
\nonl \textit{Stage III: Explore the unchosen arm in Stage II} \\
  $\mu' \leftarrow  \hat{\mu}_{1'}(t) $, $2'\leftarrow \{1,2\}\setminus {1'}$\; 
  Pull arm $ A_{t+1}=2'$ and observe reward $r_{t+1}$, $\theta_{2',1}=r_{t+1}$, $t\leftarrow t+1$, $t_2\leftarrow 1$\;
  \While{$|\mu'-\theta_{2',t_2}|<\sqrt{2/t_2\log\big(T/t_2\big(\log^2(T/t_2)+1\big)\big)}$}
 {Pull arm $A_{t+1}=2'$ and observe reward $r_{t+1}$\; $\theta_{2',t_2+1}=(t_2\theta_{2',t_2}+r_{t+1})/(t_2+1)$, $t\leftarrow t+1$, $t_2\leftarrow t_2+1$\;}
\nonl \hrulefill\\ 
\nonl \textit{Stage IV: Commit to the arm with the largest average reward } \\
 { 
    $a\leftarrow 1'\ind\{\hat{\mu}_{1'}(t)\geq\theta_{2',t_2}\}+2'\ind\{\hat{\mu}_{1'}(t)<\theta_{2',t_2}\}$\;
     \While{$t\leq {T}$}
        { 
            Pull arm $a$, $t\leftarrow t+1$\;}
        }
\end{algorithm}

Similar to Algorithm \ref{alg:double-exploration-known-gap}, Algorithm \ref{alg:double-exploration-unknown-gap} also consists of four stages, where $\stageOne$ and $\stageThree$ are double exploration stages that ensure we have chosen the right arm to pull in the subsequent stages. Since we have no knowledge about $\Delta$, we derive the stopping rule for $\stageOne$ by comparing the empirical average rewards of both arms. Once we have obtained empirical estimates of the mean rewards that are able to distinguish two arms in the sense that $|\hmu_1(t)-\hmu_2(t)|\geq\sqrt{16\log^{+}(T_1/t)/t}$, we terminate $\stageOne$. Here $t$ is the current time step of the algorithm and $T_1$ is a predefined parameter. Similar to Algorithm \ref{alg:double-exploration-known-gap}, based on the outcomes of $\stageOne$, we commit to arm $1'=\argmax_{i=1,2}\hmu_i(t)$ at the end of $\stageOne$ and pull this arm repeatedly throughout $\stageTwo$. In $\stageThree$, we turn to pull arm $2'$ that is not chosen in $\stageTwo$  
until the average reward of arm $2'$ is significantly larger or smaller than that of arm $1'$ chosen in $\stageTwo$. In $\stageFour$, we again commit to the best empirically  preforming arm and pull it till the end of the algorithm.

Compared with Algorithm \ref{alg:double-exploration-known-gap}, in both exploration stages of Algorithm \ref{alg:double-exploration-unknown-gap}, we do not use the information of the gap $\Delta$ at the cost of sequentially deciding the stopping rule in these two stages. In the following theorem, we present the regret bound of Algorithm \ref{alg:double-exploration-unknown-gap} and show that this regret is still asymptotically optimal.
\begin{theorem}
\label{theorem:unknowndelta}
Let $T_1=\log^2 T$, then the regret of Algorithm~\ref{alg:double-exploration-unknown-gap} satisfies 
\begin{align*}
    \lim_{T\rightarrow \infty} R_{\mu}(T)/\log T =2/\Delta.
\end{align*}  
\end{theorem}

The proof of Theorem \ref{theorem:unknowndelta} can be found in Section \ref{sec:proof_of_thm_unknowdelta_2arm}. Here we provide some comparison between existing algorithms and Algorithm \ref{alg:double-exploration-unknown-gap}. For two-armed bandits, \citet{lai1985asymptotically} proved that the asymptotically optimal regret  rate is $2 /\Delta$. This optimal bound has been achieved by a series of fully sequential bandit algorithms such as UCB~\citep{garivier2011kl,lattimore2018refining}, Thompson sampling~\citep{agrawal2017near}, Ada-UCB~\citep{kaufmann2018bayesian}, etc. All these algorithms are fully sequential, which means they have to examine the outcome from current pull before it can decide which arm to pull in the next time step. In contrast, DETC (Algorithm \ref{alg:double-exploration-unknown-gap}) is non-fully-sequential and separates the exploration and exploitation stages, which is much more practical in many real world applications such as clinical trials and crowdsourcing. In particular, DETC can be easily adapted to batched bandits and achieve a much smaller round complexity than these fully sequential algorithms. We will elaborate this in Section \ref{sec:batch_unknown}. 

Compared with other ETC algorithms in the unknown gap setting, ~\citet{garivier2016explore}  proved a lower bound ${4}/{\Delta}$ for `single' explore-then-commit algorithms, while the regret upper bound of DETC is improved to $2/\Delta$.  Therefore, in order to break the ${4 }/{\Delta}$  barrier in the asymptotic regret  rate, our double exploration technique in Algorithm \ref{alg:double-exploration-unknown-gap} is crucial. Different from DETC in the known gap setting, Theorem \ref{theorem:unknowndelta} does not say anything about the minimax or instance-dependent optimality of   Algorithm~\ref{alg:double-exploration-unknown-gap}. 
Because we need to guess the gap $\Delta$ during the exploration process in $\stageOne$ and $\stageThree$, additional errors may be introduced if the guess is not accurate enough. We will discuss this in details in the next section.


\subsection{Minimax and Asymptotically Optimal DETC}
If we compare the sopping rules of the exploration stages in Algorithm \ref{alg:double-exploration-known-gap} and Algorithm \ref{alg:double-exploration-unknown-gap}, we can observe that the stopping rule in the known gap setting (Algorithm \ref{alg:double-exploration-known-gap}) depends on the gap $\Delta$ (more specifically, it depends on the quantity $1/\Delta^2$ according to our analysis of the theorems in the appendix). In Algorithm \ref{alg:double-exploration-unknown-gap}, the gap $\Delta$ is unknown and guessed by the decision maker. This causes problems when the unknown $\Delta$ is too small (e.g., $\Delta=1/T^{0.1}$), where $1/\Delta^2$ is significantly large than $\log^2 T$. In this  case, after $T_1=\log^{10} T$ pulls of arm $1'$ in $\stageTwo$ of Algorithm \ref{alg:double-exploration-unknown-gap}, the average reward of $1'$ may not be close to its mean reward within a  $\Delta$ range. Hence, it fails to achieve the instance-dependent/minimax optimality.


Now we are going to show that a simple variant of Algorithm~\ref{alg:double-exploration-unknown-gap} with additional stopping rules is simultaneously minimax/instance-dependent order-optimal and asymptotically optimal.

The new algorithm is displayed in Algorithm~\ref{alg:double-exploration-unknown-gap-minimax} which has the same input, initialization, $\stageOne$ and $\stageTwo$ as Algorithm~\ref{alg:double-exploration-unknown-gap}. In \stageThree, we add an additional stopping rule $t_2<\log^2 T$ and everything else remains unchanged as in Algorithm \ref{alg:double-exploration-unknown-gap}.
The most notable change is  in $\stageFour$ of Algorithm \ref{alg:double-exploration-unknown-gap-minimax}. Instead of directly committing to the arm with the largest average reward, we will first find out how many pulls are required in $\stageThree$ to distinguish the two arms. The number of pulls in $\stageThree$ is denoted by $t_2$. If $t_2<\log^2 T$, then we just commit to the arm with the largest average reward and pull it till the end of the algorithm. However, if $t_2\geq \log^2 T$, this would mean that the gap $\Delta$ between two arms is extremely small.
In fact, we will prove that if we need to pull arm $2'$ for $\log^2 T$ times to distinguish it from arm $1'$, then with high probability the gap $\Delta$ is very small.  Consequently, we need to explore both arms again to obtain accurate estimate of their mean rewards. In a nutshell, the early stopping rule $t_2\geq\log^2T$ helps us detect the scenario with small $\Delta$ with high probability,  
which ensures the minimax/instance-dependent optimality of Algorithm~\ref{alg:double-exploration-unknown-gap-minimax}. Moreover, we will show that this condition is only violated with a tiny probability that goes to zero as $T\rightarrow\infty$, which ensures that the regret is still asymptotically optimal. 
\begin{theorem}
\label{theorem:unknowndelta-minimax}
Let $T_1=\log^{10} T$. Assume $T\Delta^2\geq 16e^3$, then the regret of Algorithm \ref{alg:double-exploration-unknown-gap-minimax} satisfies 
\begin{align*}
    \lim_{T\rightarrow \infty} R_{\mu}(T)/\log T =2/\Delta \qquad \text{ and}  \qquad  R_{\mu}(T)=O(\Delta+\log(T\Delta^2)/\Delta)=O(\Delta+\sqrt{T}).
\end{align*}  
\end{theorem}
The proof of Theorem~\ref{theorem:unknowndelta-minimax} can be found in Section \ref{sec:proof_of_two_arm_all_optimal}. 
This theorem states that Algorithm \ref{alg:double-exploration-unknown-gap-minimax}  achieves the instance-dependent/minimax and the asymptotic optimality regret simultaneously.  This is the first ETC-type algorithm that achieves these three optimal regrets simultaneously. For a two-armed bandit problem, the simultaneously instance-dependent/minimax and asymptotically optimal is also achieved by $\Delta$-UCB~\citep{garivier2016explore} and ADA-UCB~\citep{lattimore2018refining}. However, both of them are fully sequential. In contrast, our DETC algorithm shows that a non-fully-sequential algorithm can also obtain the three optimality simultaneously.  Apart from the advantages of achieving these optimalities at the same time, we emphasize that  Algorithm~\ref{alg:double-exploration-unknown-gap-minimax} also provides a  framework on how to combine an asymptotically optimal algorithm with a minimax/instance-dependent optimal algorithm. Specifically, in Algorithm~\ref{alg:double-exploration-unknown-gap-minimax}, the first part (Lines~\ref{line-alg3-1}-\ref{line-alg3-11}) of Algorithm~\ref{alg:double-exploration-unknown-gap-minimax} ensures the asymptotic optimality and the second part (Lines~\ref{line-alg3-11}-\ref{line-alg3-20}) of Algorithm \ref{alg:double-exploration-unknown-gap-minimax} ensures the minimax/instance-dependent optimality.  Following our proof of Theorem~\ref{theorem:unknowndelta-minimax} in Section \ref{sec:proof_of_two_arm_all_optimal}, one can easily verify that the second part (Lines~\ref{line-alg3-11}-\ref{line-alg3-20}) can be replaced by any other algorithm that is instance dependent optimal and  Theorem~\ref{theorem:unknowndelta-minimax} still holds. The main reason that two optimality algorithm can be combined here is that: (i): the asymptotic optimality focuses on the case that $T\rightarrow \infty$, and hence $T$ should dominate $1/\Delta$; (ii) the minimax optimality focuses on the worst case bandits  for a fixed $T$, and hence $\Delta$ could be very small (e.g., $\Delta=1/T^{0.1}$); (iii) our framework can detect if $\Delta$ is very small  via the stopping rule $t_2<\log^2 T$ in Line~\ref{line-alg3-4} of Algorithm~\ref{alg:double-exploration-unknown-gap-minimax}.


\begin{algorithm}[t]
 \small
\caption{Minimax and Asymptotically Optimal DETC in the Unknown Gap Setting}\label{alg:double-exploration-unknown-gap-minimax}
  \KwIn{$T$, $T_1$} 
 \textbf{Initialization:} Pull arms $A_1=1$, $A_2=2$, $t\leftarrow 2$; \label{line-alg3-1} \\  
\nonl \emph{Stage I: Explore all arms uniformly} \hfill(same as in Algorithm \ref{alg:double-exploration-unknown-gap})\\
\nonl \emph{Stage II: Commit to the arm with the largest average reward}
\hfill(same as in Algorithm \ref{alg:double-exploration-unknown-gap})\\
\nonl \hrulefill\\ 
\nonl \textit{Stage III: Explore the unchosen arm in Stage II} \\
$\mu'\leftarrow\hat{\mu}_{1'}(t) $, $2'\leftarrow \{1,2\}\setminus {1'}$\; 
Pull arm $ A_{t+1}=2'$ and observe reward $r_{t+1}$, $\theta_{2',1}=r_{t+1}$, $t\leftarrow t+1$, $t_2\leftarrow 1$\;
  \While{$|\mu'-\theta_{2',t_2}|<\sqrt{2/t_2\log\big(eT/t_2\big(\log^2(T/t_2)+1\big)\big)}$ and $t_2<\log^2 T$ \label{line-alg3-4}}
 {Pull arm $A_{t+1}=2'$ and observe reward $r_{t+1}$\; $\theta_{2',t_2+1}=(t_2\theta_{2',t_2}+r_{t+1})/(t_2+1)$, $t\leftarrow t+1$, $t_2\leftarrow t_2+1$\;}
\nonl \hrulefill\\ 
\nonl \textit{Stage IV: Commit to the arm with the largest average reward } \\
  \If{$t_2< \log^2 T$}
    { 
    $a\leftarrow 1'\ind\{\hat{\mu}_{1'}(t)\geq\theta_{2',t_2}\}+2'\ind\{\hat{\mu}_{1'}(t)<\theta_{2',t_2}\}$\;
     \While{$t\leq {T}$}
     { 
            Pull arm $a$, $t\leftarrow t+1$\;}
   \label{line-alg3-11} }
 \Else 
  {
   Pull arms $A_{t+1}=1$, $A_{t+2}=2$ and observe rewards $r_{t+1}$ and $r_{t+2}$\;
   $p_{1,1}=r_{t+1}$, $p_{2,1}=r_{t+2}$, $t\leftarrow t+2$, $s\leftarrow 1$\;\label{line-pis} 
   \While{$ |p_{1,s}-p_{2,s}|<\sqrt{8/s\log^{+}(T/s)}$ \label{line-alg3-18}} 
       {
          Pull arms  $A_{t+1}=1$ and $A_{t+2}=2$, and observe rewards $r_{t+1}$ and $r_{t+2}$\;
          $p_{1,s+1}=(s\cdot p_{1,s}+r_{t+1})/(s+1)$, $p_{2,s+1}=(s\cdot p_{2,s}+r_{t+2})/(s+1)$\; $t\leftarrow t+2$, $s\leftarrow s+1$\;
      }
         $a\leftarrow 1\ind\{p_{1,s}\geq p_{2,s}\}+2\ind\{p_{2,s}\geq p_{1,s}\}$\;
    \While{$t\leq {T}$}
      { 
          Pull arm $a$, $t\leftarrow t+1$. \label{line-pis-1} 
     }
 \label{line-alg3-20}}
\end{algorithm}


\section{Double Explore-then-Commit for $K$-Armed Bandits}
\label{sec:karm}
In this section, we extend our DETC framework to $K$-armed bandit problems, where $K> 2$. Due to the similarity in both structures and analyses between Algorithm \ref{alg:double-exploration-known-gap} for the known gap setting and Algorithm \ref{alg:double-exploration-unknown-gap} for the unknown gap setting, we only present the $K$-armed bandit algorithm for the unknown gap setting, which is usually more general in practice and challenging in analysis.

We present the double explore-then-commit algorithm for $K$-armed bandits in Algorithm \ref{alg:double-exploration-unknown-gap-K}. Similar to Algorithm \ref{alg:double-exploration-unknown-gap} for two-armed bandits, the algorithm proceeds as follows: (1) in $\stageOne$, we uniformly explore over all the $K$ arms; (2) in $\stageTwo$, we pull the arm with the largest average reward; (3) in $\stageThree$, we aim to ensure that the difference between the chosen arm $1'$ in $\stageTwo$ and unchosen arms is sufficient by pulling all the unchosen arm  $i'$ ($i\geq 2$) repeatedly until the average reward of arm $i'$ collected in this stage can be clearly distinguished from the average reward of arm $1'$. We set a check flag $\failflag$ initialized as $0$, which will be set to $1$ if any unchosen arm $i'$ is pulled for $\log^2T$ times; (4) in $\stageFour$, if $\failflag=0$ and $\hat\mu_{1'}$ is larger than the recalculated average reward for any other arm, then we pull  $1'$ till the end. Otherwise, $1'$ may not be the best arm.  Then we pull all arms $\log^2 T$ times, and pull the arm with the largest recalculated average reward till the end.

\begin{algorithm}[t]
 \small
   \caption{Double Explore-then-Commit  for $K$-Armed Bandits (DETC-K)}
   \label{alg:double-exploration-unknown-gap-K}
  \KwIn{$T$,  $K$.} 
  \textbf{Initialization: } $t\leftarrow 0$\;
 \nonl\hrulefill \\
 \nonl\textit{Stage I: Explore all arms uniformly} \\  
  \While{$t\leq K\sqrt{\log T}$}{Pull every arm once, $t\leftarrow t+K$\;}
   \nonl\hrulefill\\ \nonl\textit{Stage II: Commit to the arm with the largest average reward}  \\
  $1'\leftarrow \arg \max_{k} \hat{\mu}_k(t)$, $s\leftarrow 0$, $p_0\leftarrow 0$\;
  \While{$s\leq\log^2 T$}{Pull arm $A_{t+1}=1'$ and observe reward $r_{t+1}$\;
  $p_{s+1}=(s\cdot p_{s}+r_{t+1})/(s+1)$, $s\leftarrow s+1$, $t\leftarrow t+1$\;
  }
  \nonl \hrulefill\\ \nonl \textit{Stage III: Explore the unchosen arm in Stage II} \\
   $\mu'\leftarrow p_s$,  Denote $\{2',\cdots,K'\}=\{1,2,\cdots,K\}\setminus \{1'\}$\;
 \For{$i=2,3,\cdots,K$}
  {
  $t_i\leftarrow 1$,  $\theta_{i',0}=0$\;
 \While{{$|\mu'-\theta_{i',t_i}|<\sqrt{2/t_i\log\big(T/t_i\big(\log^2(T/t_i)+1\big)\big)}$ and $t_i\leq \log^2 T$}}
 {
  Pull arm $i'$ and observe reward $r_{t+1}$\;
  $\theta_{i',t_i+1}=(t_i\cdot \theta_{i',t_i}+r_{t+1})/(t_i+1)$, $t\leftarrow t+1$, $t_i\leftarrow t_i+1$\;
 }
 \If{ $t_i> \log^2 T$}
  { $\failflag\leftarrow 1$ and \textbf{break\;}
  }
 }
  \nonl\hrulefill\\ \nonl\textit{Stage IV: Commit to the arm with the largest average reward } \\
   $j':=\max_{i'} \theta_{i't_i}$\;
  \If{$\hat{\mu}_{1'}\geq\theta_{j't_j}$ and $\failflag=0$}
  {
  Let $a\leftarrow 1'$\;
 \While{$t< T$} {
  Pull arm $a$, $t\leftarrow t+1$\;}
  }
  \Else 
  { Pull every arm $\log^2 T$ times and let $a$ be the arm with the largest average reward for this pull\; \label{line:k1}
  Pull arm $a$ till $T$ time steps.\label{line:k2}
 }
\end{algorithm}

Now we present the regret bound of Algorithm \ref{alg:double-exploration-unknown-gap-K}.
\begin{theorem}\label{theorem:unknowndelta_Karm} 
The regret of Algorithm~\ref{alg:double-exploration-unknown-gap-K} with 1-subgaussian rewards satisfies
\begin{equation}\label{eq:asymptotic_rate_k}
    \lim_{T \rightarrow\infty}R_{\mu}(T)/\log (T)=\textstyle{\sum_{i: \Delta_i>0}}2/ \Delta_i.
\end{equation}
\end{theorem}

The proof of Theorem \ref{theorem:unknowndelta_Karm} can be found in Section \ref{sec:proof_of_thm_k_arm}. 
In the second case of $\stageFour$ of Algorithm \ref{alg:double-exploration-unknown-gap-K}, we actually believe that we have failed to choose the best arm via previous stages and need to explore again for a fixed number of pulls ($\log^2 T$) for all arms and commit to the best arm based on the pulling results. Note that this can be seen as the naive ETC strategy with fixed design \citep{garivier2016explore}, which has an asymptotic regret rate $4/\Delta$. Fortunately, Theorem \ref{theorem:unknowndelta_Karm} indicates our DETC algorithm can still achieve the asymptotically optimal regret for $K$-armed bandits \citep{lai1985asymptotically}.
This means that the probability of failing in the first three stages of Algorithm \ref{alg:double-exploration-unknown-gap-K} is rather small and thus the extra ETC step does not affect the asymptotic regret of our DETC algorithm.
Lastly, it would be an interesting problem to extend the idea of Algorithm \ref{alg:double-exploration-unknown-gap-minimax} in two-armed bandits to $K$-armed
bandits, where simultaneously achieving the  instance-dependent and asymptotically optimal regret is still an  open problem~\citep{agrawal2017near,lattimore2018refining}.  







\section{An Anytime Algorithm with Asymptotic Optimality}
\label{sec:anytime}
In previous sections, the stopping rules of DETC depend on the horizon length $T$. However, this may not be the case in some practical cases, where we prefer to stop the algorithm at an arbitrary time without deciding it at the beginning. This is referred to as the anytime setting in the bandit literature \citep{degenne2016anytime,lattimore2018bandit}. In this section, we provide an extension of our DETC algorithm for two-armed bandits to the anytime setting.
Our algorithm guesses $T$ in epochs. For the $r$-th epoch, we guess $T=2^{r+1}$. At the $r$-th epoch of the algorithm, the algorithm proceeds as follows: we  find the arm $1'$ which is the arm that played most often in the first $r-1$ epochs; then we pull arm $2'$ till the stopping rules  
\begin{align*}
    |\hat{\mu}_{1'}(t)-\hat{\mu}_{2'}(t)|<\sqrt{\frac{2}{T_{2'}(t)}\log\bigg(\frac{r\cdot 2^r}{T_{2'}(t)}\bigg(\log^2\bigg(\frac{r\cdot 2^r}{T_{2'}(t)}\bigg)+1\bigg)\bigg)} \qquad \text{and} \qquad t\leq 2^{r+1}
\end{align*}
is satisfied. The aim here is to ensure that the regret of pulling the winner $a(r)$ $2^{r}$ times is bounded.  After this step, we commit to the arm with the largest average reward. Compared with Algorithm~\ref{alg:double-exploration-unknown-gap},  in each epoch, anytime ETC seems only to perform the third stage and fourth stage of Algorithm~\ref{alg:double-exploration-unknown-gap}. The reason here is that: the first two stages of   Algorithm~\ref{alg:double-exploration-unknown-gap} aims to pull one arm $\log^2 T$ times while keeping the optimal regret. In anytime ETC algorithm, when the algorithm runs $\log^2 T$ steps,  $1'$ is the arm that pulled most often, thus $1'$ is pulled $O(\log^2 T)$ times. Besides, as we will prove later that the regret of first $\log^2 T$ steps is bounded by $O(\sqrt{\log T})$.

\begin{algorithm}[t]
   \caption{Anytime Asymptotically Optimal ETC in the Unknown Gap Setting}
   \label{alg:double-exploration-unknown-gap_anytime}
 \textbf{Initialization:} Pull arms $A_1=1$, $A_2=2$, $t\leftarrow 2$; \\
\For{$r=1,2,\cdots$}
 { $1' \leftarrow  \arg\max_{i\in\{1,2\}}T_{i}(t)$, $2'\leftarrow \{1,2\}\setminus {1'}$ \label{algline-anytime-choose-best-arm}\; 
  \While{$|\hat{\mu}_{1'}(t)-\hat{\mu}_{2'}(t)|<\sqrt{\frac{2}{T_{2'}(t)}\log\big(\frac{r\cdot 2^r}{T_{2'}(t)}\big(\log^2(\frac{r\cdot 2^r}{T_{2'}(t)}\big)+1\big)\big)}$ and $t\leq 2^{r+1}$
 \label{line-stopping} } 
 { $A_{t+1}=2'$, $t\leftarrow t+1$\; \label{line-alg5-5}}
    $a(r)\leftarrow 1'\ind\{\hat{\mu}_{1'}(t)\geq\hat{\mu}_{2'}(t)\}+2'\ind\{\hat{\mu}_{1'}(t)<\hat{\mu}_{2'}(t)\}$\;
     \While{$t\leq 2^{r+1}$
       }
        { 
            Pull arm $a(r)$, $t\leftarrow t+1$\; \label{line-alg5-8}}
}
\end{algorithm}

The following theorem shows that Algorithm \ref{alg:double-exploration-unknown-gap_anytime} is still asymptotically optimal for an unknown horizon $T$. 

\begin{theorem}
\label{thm:anytime}
The total expected regret for the anytime version of DETC (Algorithm \ref{alg:double-exploration-unknown-gap_anytime}) satisfies
$\lim_{T\rightarrow \infty}R_{\mu}(T)/\log T=2/\Delta$.
\end{theorem}

The proof of Theorem \ref{thm:anytime} could be found in Section~\ref{sec:detc-anytime}. The result shows that even for the anytime setting (unknown horizon length), the ETC strategy can also be asymptotically optimal as  UCB~\citep{katehakis1995sequential} and Thompson Sampling~\citep{korda2013thompson} do. An advantage of the anytime ETC algorithm is that Algorithm \ref{alg:double-exploration-unknown-gap_anytime} only needs $O(\log T)$ epochs, where in each epoch it can separate exploration and exploitation stages, while for anytime UCB or Thompson Sampling algorithms often need $O(T)$ mixed exploration and exploitation stages.


\section{Asymptotically Optimal DETC in Batched Bandit Problems}\label{sec:batch}
The proposed DETC algorithms in this paper can be easily extended to batched bandit problems \citep{perchet2016batched,NIPS2019_8341,esfandiari2019batched}. In this section, we present simple modifications to Algorithms \ref{alg:double-exploration-known-gap} and  \ref{alg:double-exploration-unknown-gap} which we refer to as \emph{Batched DETC}. We prove that they not only achieve the asymptotically optimal regret bounds  but also enjoy $O(1)$ round complexities.

\subsection{Batched DETC in the Known Gap Setting}\label{sec:batch_known}
We use the same notations that are used  in Section \ref{sec:detc_known}. The Batched DETC algorithm is identical to Algorithm \ref{alg:double-exploration-known-gap} except the stopping rule of $\stageThree$. More specifically,
let $\tau_0=\log(T\Delta^2)/(2(1-\epsilon_T)^2\Delta^2)$. In $\stageThree$ of Algorithm \ref{alg:double-exploration-known-gap}, instead of querying the result $\theta_{2',t_2}$ at every step $t_2=0,1,\ldots$, we only query it at the following time grid: 
\begin{equation}
\label{eq:timenodes}
    \cT=\bigg\{\bigg\lceil \tau_0+\frac{2\sqrt{\log(T\Delta^2)}+4}{2(1-\epsilon_T)^2\Delta^2} \bigg\rceil, \bigg \lceil \tau_0+\frac{2(2\sqrt{\log(T\Delta^2)}+4)}{2(1-\epsilon_T)^2\Delta^2} \bigg \rceil, \ \bigg \lceil \tau_0+\frac{3(2\sqrt{\log(T\Delta^2)}+4)}{2(1-\epsilon_T)^2\Delta^2} \bigg \rceil, \cdots  \bigg\}.
\end{equation}
At each time point $t_2\in\cT$, we query the results of the bandits pulled since the last time point. Between two time points, we pull the arm $2'$ without accessing the results. The period between two times points is also referred to as a round~\citep{perchet2016batched}. Reducing the total number of queries, namely, the round complexity, is an important research topic in the batched bandit problem. For the convenience of readers, we present the Batched DETC algorithm for known gaps in Algorithm \ref{alg:batched_DETC_known_gap}. Note that in this batched version, $\stageOne$, $\stageTwo$ and $\stageFour$ of Algorithm \ref{alg:batched_DETC_known_gap} are identical to that of Algorithm \ref{alg:double-exploration-known-gap}, and we omit them for the simplicity of presentation.

\begin{algorithm}[tbp]
\caption{Batched DETC in the Known Gap Setting}\label{alg:batched_DETC_known_gap}
  \KwIn{$T$, $\epsilon_T$,  $\Delta$ and $\cT$ defined in \eqref{eq:timenodes}} 
\textbf{Initialization:} Pull arms $A_1=1$ and $A_2=2$, $t\leftarrow 2$, $T_1= \lceil 2\log (T\Delta^2)/(\epsilon_T^2\cdot \Delta^2)\rceil$, $\tau_1=4\lceil\log(T_1\Delta^2)/\Delta^2\rceil$\;  
  \nonl \emph{Stage I: Explore all arms uniformly} \hfill(same as in Algorithm \ref{alg:double-exploration-known-gap})\\
  \nonl \emph{Stage II: Commit to the arm with the largest average reward} \hfill(same as in Algorithm \ref{alg:double-exploration-known-gap})\\
  \nonl\hrulefill\\ \nonl \textit{Stage III: Explore the unchosen arm in Stage II} \\
   $\mu'\leftarrow \hat{\mu}_{1'}(t) $, $t_2\leftarrow 0$, $2'\leftarrow \{1,2\}\setminus {1'}$, $\theta_{2',s}$ is the recalculated average reward of arm $2'$ after its $s$-\emph{th} pull  in Stage \emph{III} and $\theta_{2's}\leftarrow 0$, for $s=0$\;
  \While{\textbf{true}}
  {
  \If{$t_2\in\cT$}
     {
  \If{${2(1-\epsilon_T)t_2\Delta}\mid \mu'-\theta_{2',t_2} \mid \geq\log(T\Delta^2)$}
     {
  \textbf{break}\;
     }
     }
   Pull arm $A_{t+1}=2'$, $t\leftarrow t+1$, $t_2\leftarrow t_2+1$\; 
 }
 \nonl \emph{Stage IV: Commit to the arm with the largest average reward} \hfill{(same as in Algorithm \ref{alg:double-exploration-known-gap})} \\
\end{algorithm}

Now, we present the round complexity of  the Batched DETC in Algorithm~\ref{alg:batched_DETC_known_gap}.
\begin{theorem}\label{thm:knowngap_batch}
In the batched bandit problem, the expected number of rounds used in Algorithm~\ref{alg:batched_DETC_known_gap} is $O(1)$. At the same time, the regret of Algorithm \ref{alg:batched_DETC_known_gap} is asymptotically optimal.
\end{theorem}

\begin{remark}
The proof of Theorem \ref{thm:knowngap_batch} can be found in Section \ref{sec:proof_batch_detc_known_gap_round}. Compared with fully sequentially adaptive bandit algorithms such as UCB, which needs $O(T)$ rounds of queries, our DETC algorithm only needs constant rounds of queries (independent of the horizon length $T$). Compared with another constant round algorithm FB-ETC in \citep{garivier2016explore}, our DETC algorithm simultaneously  improves the asymptotic regret rate of FB-ETC (i.e., $4/\Delta$) by a factor of $8$.
\end{remark}

\subsection{Batched DETC in the Unknown Gap Setting}\label{sec:batch_unknown}
In the unknown gap setting, both the stopping rules of $\stageOne$ and $\stageThree$ in Algorithm \ref{alg:double-exploration-unknown-gap} need to be modified. In what follows, we describe a variant of  Algorithm~\ref{alg:double-exploration-unknown-gap} that only needs to check the results of pulls at certain time points in $\stageOne$ and $\stageThree$. In particular, let $T_1=\log^{2}T$. In $\stageOne$, we query the results and test the condition in Line 3 
of Algorithm \ref{alg:batched_DETC_unknown_gap} at the following time grid: 
\begin{equation}
\label{eq:testtt}
    t\in\cT_{2}=\{2\sqrt{\log T}, 4\sqrt{\log T}, 6\sqrt{\log T},\ldots\}. 
\end{equation}
In $\stageThree$, we the condition in Line 10 of Algorithm \ref{alg:batched_DETC_unknown_gap} is only checked at the following time grid. 
\begin{align}
\label{eq:testt_2}
\begin{split}
    t_2\in\cT_2'=\big\{N_1, &2/\hat{\Delta}^2 N_2\log (T \log^3 T)+1/\hat{\Delta}^2 N_2(\log T)^{\frac{2}{3}},\\
    &2/\hat{\Delta}^2 N_2\log (T \log^3 T)+2/\hat{\Delta}^2 N_2(\log T)^{\frac{2}{3}},\\
    &2/\hat{\Delta}^2N_2\log (T \log^3 T)+3/\hat{\Delta}^2 N_2(\log T)^{\frac{2}{3}}, \cdots, \log^{2} T \big\}.
\end{split}
\end{align}
where $N_1=(2\log T)/\log \log T$,  $N_2=(1+(\log T)^{-\frac{1}{4}})^2$, and $\hat{\Delta}= |\mu'-\theta_{2',N_1}|$ is an estimate of $\Delta'$ based on the test result after the first round (the first $N_1$ steps).
Apart from restricting $t_2\in\cT_2'$, another difference here from Algorithm \ref{alg:double-exploration-unknown-gap} is that we require $t_2 \leq \log^2 T$. Thus we will terminate $\stageThree$ after at most $\log^2 T$ pulls of arm $2'$. For the convenience of readers, we display the modified Algorithm \ref{alg:double-exploration-unknown-gap} for batched bandits with  unknown gaps in Algorithm \ref{alg:batched_DETC_unknown_gap}. 

\begin{algorithm}[h]
\caption{Batched DETC in the Unknown Gap Setting}\label{alg:batched_DETC_unknown_gap}
  \KwIn{$T$, $T_1$, $\cT_2$ defined in \eqref{eq:testtt}, and $\cT_2'$ defined in \eqref{eq:testt_2}} 
 \textbf{Initialization:} Pull arms $A_1=1$, $A_2=2$, $t\leftarrow 2$\;
 \nonl \hrulefill \\
 \nonl \textit{Stage I: Explore all arms uniformly}\\  
  \While{ {\bf true} } 
  {
  \If{$t\in\cT_2$} 
  {
  \If{$\mid \hmu_1(t)-\hat{\mu}_2(t)\mid \geq\sqrt{16/t\log^{+}(T_1/t)}$}
  { {\bf break}\;
  }
  }
   Pull arms  $A_{t+1}=1$ and $A_{t+2}=2$, $t\leftarrow t+2$\;
  } 
  \nonl \emph{Stage II: Commit to the arm with the largest average reward} \hfill(same as in Algorithm \ref{alg:double-exploration-unknown-gap})\\
  \nonl \hrulefill\\ \nonl \textit{Stage III: Explore the unchosen arm in Stage II} \;
  $\mu' \leftarrow  \hat{\mu}_{1'}(t) $, $2'\leftarrow \{1,2\}\setminus {1'}$, $t_2\leftarrow 0$,  $\theta_{2's}$ is the  recalculated average reward of arm $2'$ after its $s$-\emph{th} pull in Stage \emph{III} and $\theta_{2's}\leftarrow 0$, for $s=0$\;
 \While{$t_2\leq \log^2 T$}
 {
  \If{$t_2\in\cT'_2$}
  {
  \If{$|\mu'-\theta_{2',t_2}|<\sqrt{2/t_2\log\big(T/t_2\big(\log^2(T/t_2)+1\big)\big)}$}
  {
   \textbf{break}\;
  }
  }
 Pull arm $A_{t+1}=2'$, $t\leftarrow t+1$, $t_2\leftarrow t_2+1$\;
 } 
  \nonl\emph{Stage IV: Commit to the arm with the largest average reward}  \hfill(same as in Algorithm \ref{alg:double-exploration-unknown-gap})\\
\end{algorithm}

\begin{theorem}\label{thm:unknowngap_batch}
In the batched bandit problem, the expected number of rounds used in Algorithm~\ref{alg:batched_DETC_unknown_gap} is $O(1)$. Moreover, the regret of Algorithm \ref{alg:batched_DETC_unknown_gap} is asymptotically optimal. 
\end{theorem}
The proof of Theorem \ref{thm:unknowngap_batch} can be found in Section \ref{sec:proof_batch_detc_unknown_gap_round}.
Here, we only focus on deriving the asymptotic optimality along with a constant round complexity in the batched bandits setting. For minimax and instance dependent regret bounds, \cite{perchet2016batched} proved that any algorithm achieving the minimax optimality or instance dependent optimality will cost at least $\Omega(\log \log T)$ or $\Omega(\log T/\log \log T)$ rounds respectively. How to extending our minimax/instance-dependent and asymptotic optimal Algorithm~\ref{alg:double-exploration-unknown-gap-minimax}  to the batched bandit setting is an interesting open question.


\section{Conclusion}
\label{sec:future-work}



In this paper, we revisit the explore-then-commit (ETC) type of algorithms for multi-armed bandit problems, which separate the exploration and exploitation stages. We break the barrier that ETC type algorithms cannot achieve the asymptotically optimal regret bound \citep{garivier2016explore}, which is usually attained by fully sequential strategies such as UCB. We propose a double explore-then-commit (DETC) strategy and prove that DETC is asymptotically optimal for subgaussian rewards, which is the first ETC type algorithm that matches the theoretical performance of UCB based algorithms. We also show a variant of DETC for two-armed bandit problems, which can achieve the asymptotic optimality and the minimax/instance-dependent regret bound simultaneously. To demonstrate the advantage of DETC over fully sequential strategies, we apply DETC to the batched bandit problem which has various of real world applications and prove that DETC enjoys a constant round complexity while maintaining the asymptotic optimality at the same time. As a comparison, the round complexity of fully sequential strategies usually scales with the horizon length $T$ of the algorithm. This implies that the proposed DETC algorithm not only enjoys optimal regret bounds under various metrics, but is also practical and easily implementable in applications where the decision maker is expected to not switch its policy frequently.

\appendix
\section{Proof of the Regret Bound of Algorithm \ref{alg:double-exploration-known-gap}}\label{sec:proof_of_thm_knowdelta_2arm}
Now we are going to prove Theorem \ref{theorem:knowndelta_2arm}. We first present a technical lemma that characterizes the concentration properties of subgaussian random variables.
\begin{lemma}[Corollary 5.5 in \cite{lattimore2018bandit}]
\label{lem:subguassian}
Assume that $X_1,\ldots,X_n$ are independent, $\sigma$-subguassian random variables centered around $\mu$. Then for any $\epsilon>0$
\begin{align}
    \mathbb{P}(\hat{\mu}\geq \mu+\epsilon)\leq \exp \bigg(
    -\frac{n\epsilon^2}{2\sigma^2}\bigg) \quad\text{and } \ \ \ \ \ \ \  \PP(\hat{\mu}\leq \mu-\epsilon)\leq \exp \bigg(
    -\frac{n\epsilon^2}{2\sigma^2}\bigg),
\end{align}
where $\hat{\mu}=1/n\sum_{t=1}^n X_t$.
\end{lemma}
\begin{proof}[Proof of Theorem \ref{theorem:knowndelta_2arm}]
Let $\tau_2$ be the total number of times arm $2'$ is pulled in \emph{Stage III} of Algorithm \ref{alg:double-exploration-known-gap}. We know that $\tau_2$ is a random variable. 
Recall that $\mu_1>\mu_2$ and $\Delta=\mu_1-\mu_2$. Recall $\tau_1$ is number of times arm 1 is pulled in \emph{Stage I}.  Let $N_2(T)$ denote the total number of times Algorithm \ref{alg:double-exploration-known-gap} pulls arm $2$, which is calculated as
\begin{align}\label{eq:no_pull_arm_decomp}
    N_2(T)
    =\tau_1&+(T_1-\tau_1)\ind\{\hat\mu_1(\tau_1)<\hat\mu_2(\tau_1)\}+\tau_2\ind\{\hat\mu_1(\tau_1)\geq\hat\mu_2(\tau_1)\}\notag\\
    &{}+(T-T_1-\tau_1-\tau_2)\ind\{a=2\}.
\end{align}
Then, the regret of Algorithm \ref{alg:double-exploration-known-gap} $R_{\mu}(T)=\EE[\Delta N_2(T)]$ can be decomposed as follows 
\begin{align}\label{eq:alg1gre}
    R_{\mu}(T)
    &\leq\EE\big[\Delta\tau_1+\Delta(T_1-\tau_1)\ind\{\hat\mu_1(\tau_1)<\hat\mu_2(\tau_1)\}+\Delta\tau_2 \ind\{\hat\mu_1(\tau_1)\geq\hat\mu_2(\tau_1)\ind\}+\Delta T\ind\{a=2\}\big]\notag\\
    &\leq\EE\big[\Delta\tau_1+\Delta T_1\PP(\hat\mu_1(\tau_1)<\hat\mu_2(\tau_1))+\Delta\tau_2 \PP(\hat\mu_1(\tau_1)\geq\hat\mu_2(\tau_1))+\Delta T\PP(a=2)\big]\notag\\
    &\leq\Delta\tau_1+\underbrace{\Delta T_1 \mathbb{P}(\tau_1<T_1,1'=2)}_{I_1}
  +\underbrace{\Delta\mathbb{E}[{\tau_2}]}_{I_2}+ \underbrace{\Delta T\mathbb{P}(\tau_2<T,a=2)}_{I_3}.  
\end{align}
In what follows, we will bound these terms separately. \\

\noindent\textbf{Bounding term $I_1$:} Let $X_i$  and $Y_i$ be the rewards from pulling arm 1 and  arm 2 for the $i$-th time respectively. Thus $X_i-\mu_1$ and $Y_i-\mu_2$ are 1-subgaussian random variables.
Let $S_0=0$ and $S_n=(X_1-Y_1)+\cdots+(X_n-Y_n)$ for every $n\geq 1$. Then $X_i-Y_i-\Delta$ is a $\sqrt{2}$-subgaussian random variable. 
Applying Lemma~\ref{lem:subguassian} with any $\epsilon>0$, we get
\begin{equation}
     \mathbb{P}(S_{\tau_1}/\tau_1\leq \Delta-\epsilon) \leq \exp(-\tau_1\epsilon^2/4)\leq\exp(-\epsilon^2\log(T_1\Delta^2)/\Delta^2)  ,
\end{equation}
where in the last inequality we plugged in the fact that $\tau_1\geq 4\log(T_1\Delta^2)/\Delta^2$. By setting $\epsilon=\Delta$ in the above inequality, we further obtain
$\PP(\tau_1<T_1,1'=2)= \PP (S_{\tau_1}/\tau_1\leq 0)\leq 1/(T_1\Delta^2)$. Hence
\begin{equation}\label{eq:com12}
   I_1= T_1\Delta \mathbb{P}(\tau_1<T_1,1'=2) \leq 1/\Delta.
\end{equation}

\noindent\textbf{Bounding term $I_2$:}  
Recall that $T_1\geq 2\log(T\Delta^2)/(\epsilon_T^2\Delta^2)$. Define event $E=\{\mu'\in (\mu_{1'}-\epsilon_T\Delta, \mu_{1'}+\epsilon_T\Delta)\}$, and let $E^c$ be the complement of $E$.  By Lemma~\ref{lem:subguassian} and the union bound, $\PP(E)\geq 1-2/(T\Delta^2)$.  Therefore, \begin{align}
\label{eq:bound-I_2-added}
    I_2&=\Delta\EE[\tau_2\ind(E)]+\Delta\EE[\tau_2\ind(E^c)]\notag\\
    &=\Delta\EE[\tau_2\ind(E)]+\Delta\EE[\tau_2\mid E^c] \cdot \PP(E^c)\notag\\
    &\leq \Delta\EE[\tau_2\ind(E)]+\Delta T\cdot \frac{2}{T\Delta^2} \notag \\
    & =\Delta\EE[\tau_2\ind(E, 1'=1)]+\Delta\EE[\tau_2 \ind(E, 1'=2)]+2/\Delta.
\end{align}

 We first focus on term $\Delta\EE[\tau_2\ind(E, 1'=1)]$. 
 Observe that when $E$ holds and $1'=1$ (i.e., the chosen arm $1'$ is the best arm),  arm $2'=2$ is pulled in \emph{Stage III} of Algorithm \ref{alg:double-exploration-known-gap}. For ease of presentation, we define the following notations:
\begin{align}\label{eq:2arm_knowngap_def_Sn_process}
  Z_0=0,\quad Z_i=\mu'-Y_{i+\tau_1},\quad
  S'_0=0,\quad
  S'_n=Z_1+\cdots+Z_n,
\end{align}
where $Y_{i+\tau_1}$ is the reward from pulling arm $2$ for the $i$-\emph{th} time in Stage \emph{III}. 
For any $x>0$, we define $$n_{x}=({\log(T\Delta^2)+x})/({2(1-\epsilon_T)^2\Delta^2}).$$ 
We also define a check point parameter $x_0=2\sqrt{\log (T\Delta^2)}$.  

Let $E_1$ denote the event $\{E, 1'=1\}$. Note that in \emph{Stage III} of Algorithm \ref{alg:double-exploration-known-gap}, conditioned on $E_1$, we have 
\begin{align*}
    2(1-\epsilon_T)\Delta |S_{t_2}'|=2(1-\epsilon_T)t_2\Delta|\mu'-\theta_{2',t_2}|<\log(T\Delta^2),
\end{align*}
for $t_2\leq \tau_2-1$. Therefore, conditioned on $E_1$,
\begin{align}\label{eq:2arm_knowngap_event_tau2}
     \bigg\{\tau_2-1 \geq \bigg\lceil \frac{\log(T\Delta^2)+x}{2(1-\epsilon_T)^2\Delta^2} \bigg\rceil \bigg\} & =\{\tau_2-1\geq \lceil n_x\rceil \}
     \notag \\
     & \subseteq \bigg\{ S_{\lceil n_x \rceil}' 
     \leq \frac{\log(T\Delta^2)}{2(1-\epsilon_T)\Delta}  \bigg\}.
\end{align}
Let $\Delta'=\mu'-\mathbb{E}[Y_{i+\tau_1}]$. 
Then, $Z_i-\Delta'$ is 1-subgaussian. We have that conditioned on $E_1$,
\begin{equation}
     \Delta'=\mu'-\mathbb{E}[Y_{1+\tau_1}]=\mu'-\mu_2\geq \mu_1-\epsilon_T\Delta-\mu_2=(1-\epsilon_T)\Delta.
\end{equation}
 By Lemma~\ref{lem:subguassian}, for any $\epsilon>0$, we have 
\begin{equation}
 \label{eq:S_nx}
\begin{split}
 \mathbb{P} \left(\frac{S_{\lceil n_x \rceil}'}{\lceil n_x \rceil} \leq \Delta'-\epsilon \;\middle|\; E_1\right) & \leq \exp\left(-\lceil n_x \rceil\epsilon^2/2\right)  .
\end{split}
\end{equation}
Let $\epsilon=\frac{(1-\epsilon_T)\Delta x}{\log(T\Delta^2)+x}$. Conditioned on $E_1$,
$$\lceil n_x \rceil(\Delta'-\epsilon)\geq \lceil n_x \rceil( (1-\epsilon_T)\Delta-\epsilon )\geq \frac{\log(T\Delta^2)}{2(1-\epsilon_T)\Delta}.$$ 
Combining this with \eqref{eq:S_nx} yields
\begin{align}
\label{eq:roundused1}
    \mathbb{P}\left(S_{\lceil n_x \rceil}'\leq \frac{\log(T\Delta^2)}{2(1-\epsilon_T)\Delta} \;\middle|\; E_1 \right) & \leq  \PP\left(S_{\lceil n_x \rceil}'\leq \lceil n_x \rceil(\Delta'-\epsilon) \;\middle|\; E_1 \right) \notag \\
    & \leq \exp\bigg(-\frac{x^2}{4(\log(T\Delta^2)+x)}\bigg). 
\end{align}
This, when combined with \eqref{eq:2arm_knowngap_event_tau2}, implies
\begin{align*}
    \PP\left(\tau_2-1\geq\bigg\lceil \frac{\log(T\Delta^2)+x}{2(1-\epsilon_T)^2\Delta^2} \bigg\rceil \;\middle|\; E_1\right)\leq\exp\bigg(-\frac{x^2}{4(\log(T\Delta^2)+x)}\bigg).
\end{align*}
Recall that $x_0=2\sqrt{\log (T\Delta^2)}$. For any $x\geq  x_0$, we have $x\sqrt{\log(T\Delta^2)}/2\geq\log(T\Delta^2)$. 
Thus, 
\begin{align}\label{eq:geqyregret}
\int_{n_{x_0}}^{\infty}\mathbb{P}(\tau_2-2\geq v\mid E_1) \dd v
& =\int_{x_0}^{\infty}\mathbb{P}\left( \tau_2-2\geq \frac{\log(T\Delta^2)+x}{2(1-\epsilon_T)^2\Delta^2} \;\middle|\; E_1 \right)\frac{\dd x}{2(1-\epsilon_T)^2\Delta^2} \notag \\
& \leq\int_{x_0}^{\infty}\mathbb{P}\left( \tau_2-1\geq \bigg\lceil \frac{\log(T\Delta^2)+x}{2(1-\epsilon_T)^2\Delta^2}  \bigg\rceil \;\middle|\; E_1 \right)\frac{\dd x}{2(1-\epsilon_T)^2\Delta^2} \notag \\
& \leq \frac{1}{2(1-\epsilon_T)^2\Delta^2} \int_{x_0}^{\infty} \exp\bigg(-\frac{x^2}{4(\log(T\Delta^2)+x)}\bigg) \dd x \notag \\
& \leq \frac{1}{2(1-\epsilon_T)^2\Delta^2} \int_{x_0}^{\infty} \exp \bigg( -\frac{x}{2\sqrt{\log(T\Delta^2)}+4}\bigg) \dd x \notag \\
& \leq \frac{1}{2(1-\epsilon_T)^2\Delta^2} \int_{0}^{\infty} \exp \bigg( -\frac{x}{2\sqrt{\log(T\Delta^2)}+4}\bigg) \dd x \notag \\
& =\frac{\sqrt{\log(T\Delta^2)}+2}{(1-\epsilon_T)^2\Delta^2}.
\end{align}
Then, the expectation of $\Delta\tau_2$ conditioned on $E_1$ is 
\begin{align}\label{eq:regrettau_2}
\Delta\EE[\tau_2\mid E_1] &=\Delta\int_{0}^{\infty}\PP(\tau_2>v \mid E_1)\dd  v\notag\\
&= \Delta\int_{0}^{n_{x_0}+2}\PP(\tau_2>v \mid E_1)\dd  v+\Delta\int_{n_{x_0}}^{\infty}\mathbb{P}( \tau_2-2\geq v \mid E_1) \dd v\notag\\ 
&\leq 2\Delta+\frac{\log(T\Delta^2)}{2(1-\epsilon_T)^2\Delta}+\frac{2\sqrt{\log(T\Delta^2)}+2}{(1-\epsilon_T)^2\Delta}.
\end{align}
Hence, we have
\begin{align}
    \Delta\EE[\tau_2\ind(E, 1'=1)]& 
    =\Delta\EE[\tau_2\mid E_1] \cdot \PP(E_1) \notag\\
    & \leq  \PP(E_1)\cdot \bigg( 2\Delta+\frac{\log(T\Delta^2)}{2(1-\epsilon_T)^2\Delta}+\frac{2\sqrt{\log(T\Delta^2)}+2}{(1-\epsilon_T)^2\Delta}\bigg).
\end{align}

Let $E_2$ denote the event $\{E, 1'=2\}$. In a manner similar to the proof of~\eqref{eq:regrettau_2}, we can show that 
\begin{align}
    \Delta\EE[\tau_2\ind(E, 1'=2)]& 
    =\Delta\EE[\tau_2\mid E_2] \cdot \PP(E_2) \notag \\
    & \leq  \PP(E_2)\cdot \bigg( 2\Delta+\frac{\log(T\Delta^2)}{2(1-\epsilon_T)^2\Delta}+\frac{2\sqrt{\log(T\Delta^2)}+2}{(1-\epsilon_T)^2\Delta}\bigg).
\end{align}

Therefore, we have
\begin{align}
\label{eq:sodafinal-1}
    I_2&\leq\Delta\EE[\tau_2\ind(E, 1'=1)]+\Delta\EE[\tau_2\ind(E, 1'=2)]+\frac{2}{\Delta} \notag \\
    & \leq 2\Delta+\frac{2}{\Delta}+\frac{\log(T\Delta^2)}{2(1-\epsilon_T)^2\Delta}+\frac{2\sqrt{\log(T\Delta^2)}+2}{(1-\epsilon_T)^2\Delta}.
\end{align}

\noindent\textbf{Bounding term $I_3$:} 
For term $I_3$, similar to~\eqref{eq:bound-I_2-added}, we have
\begin{align}
    I_3 = & \Delta \cdot T\PP[\tau_2<T,a=2 \mid E_1]\cdot \PP[E_1] \notag\\
    & {} +\Delta \cdot T\PP[\tau_2<T,a=2 \mid E_2]\cdot \PP[E_2]+\frac{2}{\Delta}.
\end{align} 
We will first prove that ${\mathbb{P}(\tau_2<T, a=2 \mid E_1)\leq {1}/({T\Delta^2})}$. Recall that  $S_n'=\sum_{i=1}^{n}Z_i$ and $Z_i=\mu'-Y_{i+\tau_1}$. In addition, $Z_i-\Delta'$ is 1-subgaussian, and $\Delta'\geq (1-\epsilon_T)\Delta$ whenever  $E_1$ occurs. Then, 
\begin{align}
    \mathbb{E}[\exp(-2\Delta(1-\epsilon_T)Z_1) \mid E_1] & =   \mathbb{E} [\exp(-2\Delta(1-\epsilon_T)Z_1+2\Delta\Delta'(1-\epsilon_T) -2 \Delta\Delta'(1-\epsilon_T) ) \mid E_1] \notag\\
    & =\mathbb{E} [\exp(-2\Delta(1-\epsilon_T)(Z_1-\Delta') -2 \Delta\Delta'(1-\epsilon_T)) \mid E_1 ] \notag \\ 
    & \leq \exp((-2(1-\epsilon_T)\Delta)^2/2 -2(1-\epsilon_T)\Delta\Delta')) \notag\\
    &\leq \exp(2(1-\epsilon_T)\Delta((1-\epsilon_T)\Delta-\Delta')) \notag\\
    & \leq 1,
\end{align}
where the first inequality follows from the definition of subgaussian random variables. 
We consider the sigma-algebra $F_n =\sigma(E_1, Y_{\tau_1+i}, i=1,...,n)$ for $n \geq 1$. Define $F_0={E_1}$ and $M_0=1$. Then, the sequence $\{M_n\}_{n=0, 1, ...}$ with 
$M_n=\exp (-2\Delta(1-\epsilon_T)S'_n)$ is a super-martingale with respect to $\{F_n\}_{n=0, 1, ...}$.   Let $\tau'=T\wedge \inf\{n> 1: S'_n\leq  -{\log (T\Delta^2)}/({2\Delta(1-\epsilon_T)})\}$ be a stopping time. Observe that conditioned on $E_1$,

\begin{align}
    \{\tau_2<T, a=2 \} & \subseteq \left\{\exists 1<n< T: S'_n\leq -\frac{\log(T\Delta^2)}{2\Delta(1-\epsilon_T)}  \right\} \notag \\
    & =\{\tau'<T \}.
\end{align}
Applying Doob's optional stopping theorem \citep{durrett2019probability} yields $\mathbb{E}[M_{\tau'}]\leq \mathbb{E}[M_0]=1$. In addition, when $\tau_2<T$, we have
\begin{align}
M_{\tau'} & =\exp(-2\Delta(1-\epsilon_T)S'_{\tau'})  \geq \exp(\log(T\Delta^2)) =T\Delta^2.
\end{align}
In other words, $\{\tau_2 < T\} \subseteq \{M_{\tau'} \ge T\Delta^2\}$. This leads to
\begin{align}
    \mathbb{P}(\tau_2<T,a=2 \mid E_1) & \leq \mathbb{P}(\tau'<T \mid E_1) \notag \\ & \leq \PP(M_{\tau'}\geq T\Delta^2  \mid E_1) \notag \\ 
    & \leq \mathbb{E}[M_{\tau'}]/(T\Delta^2)\notag\\
    &\leq 1/(T\Delta^2).
\label{eq:com13}    
\end{align} 
where the third inequality follows form Markov's inequality. Similarly,  $\mathbb{P}(\tau_2<T,a=2 \mid E_2)\leq {1}/({T\Delta^2})$ also holds. Thus, term $I_3$ can be upper bounded by $3/\Delta$. \\

\noindent\textbf{Completing the proof:}  Substituting~\eqref{eq:com12}, \eqref{eq:sodafinal-1} and $I_3\leq 3/\Delta$ into \eqref{eq:alg1gre} yields a total regret as follows
\begin{equation}
   R_{\mu}(T)\leq 2\Delta+ \frac{8}{\Delta}+\frac{4\log(T_1\Delta^2)}{\Delta}+\frac{\log(T\Delta^2)+2\sqrt{\log(T\Delta^2)}}{2(1-\epsilon_T)^2\Delta}+\frac{\sqrt{\log(T\Delta^2)}+2}{(1-\epsilon_T)^2\Delta}.\notag
\end{equation}
Recall the choice of $\epsilon_T$ in Theorem \ref{theorem:knowndelta_2arm}. By our choice that $T_1=\lceil 2\log(T\Delta^2)/(\epsilon_T^2 \Delta^2)) \rceil$, we have
\begin{align}\label{eq:2arm_knowngap_T1_bound}
    T_1\leq 1+ \max\{2\log^2T,8\log(T\Delta^2)/\Delta^2\},
\end{align}
which immediately implies, $\lim_{T\rightarrow \infty}4 {\log(T_1\Delta^2)}/(\Delta \log T) =0$. Also note that $\lim_{T\rightarrow\infty}\epsilon_T=0$. Thus, we have $\lim_{T\rightarrow\infty} R_{\mu}(T)/\log
T=1/(2\Delta)$. By \eqref{eq:2arm_knowngap_T1_bound}, we known that $T_1\Delta^2=O(\log (T\Delta^2))$, which results in the worse case regret bound as
\begin{align*}
    R_{\mu}(T)=O\bigg(\Delta+\frac{1}{\Delta}+\frac{\log(T\Delta^2)}{\Delta}+\frac{\log\log(T\Delta^2)}{\Delta} \bigg)=O(\Delta+\frac{\log(T\Delta^2)}{\Delta})=O(\Delta+\sqrt{T}),
\end{align*}
where the last equality is due to the fact that $T\Delta^2>1$ and $\log x\leq 2\sqrt{x}$ for $x>1$.

\end{proof}

\section{Proof of the Regret Bound of Algorithm \ref{alg:double-exploration-unknown-gap}}\label{sec:proof_of_thm_unknowdelta_2arm}
Next, we provide the proof for Theorem \ref{theorem:unknowndelta}. Note that the stopping time of $\stageOne$ and {\it Stage III} in Algorithm~\ref{alg:double-exploration-unknown-gap} is not fixed and instead depends on the random samples, and hence, the Hoeffding's inequality in Lemma~\ref{lem:subguassian} is not directly applicable. To address this issue, we provide the following two Lemmas.

\begin{lemma}\label{lemma:maximal_ineq}
Let $N$ and $M$ be extended real numbers in $\mathbb{R}^{+}$ and $\mathbb{R}^{+}\cup\{+\infty\}$. Let $\gamma$ be a real number in $\mathbb{R}^+$, and let $\hat{\mu}_n=\sum_{s=1}^n X_s/n$ be the empirical mean of $n$ random variables identically independently distributed according to 1-subgaussian distribution. Then 
\begin{align}\label{eq:maximal_ineq}
    \PP(\exists N\leq n\leq M, \hat{\mu}_n+\gamma\leq 0) \leq \exp\bigg(-\frac{N\gamma^2}{2}\bigg).
\end{align}
\end{lemma}

The following lemma characterizes the length of the uniform exploration in $\stageOne$ of Algorithm \ref{alg:double-exploration-unknown-gap}. Since each arm is pulled for the same number of times (e.g., $s$ times), the length of $\stageOne$ is $2s$.
\begin{lemma}
\label{lem:peeling-simp}
Let $n\in \NN^+$,  $X_1,X_2,\cdots$, be  i.i.d.  1-subgaussian random variables, and $Y_1,Y_2,\cdots$, be  i.i.d. 1-subgaussian random variables. Assume without loss of generality that $\EE[X_1]>\EE[Y_1]$. Denote $\Delta=\EE[X_i-Y_i]$, and  $\hat{\mu}_{t}=1/\sum_{n=1}^t (X_n-Y_n)$. Then for any $x>0$,
\begin{align*}
    \PP \bigg(\exists s\geq 1: \hat{\mu}_{s}+\sqrt{\frac{8}{s}\log^{+}\bigg(\frac{N}{s}\bigg)}\leq 0  \bigg)\leq \frac{15}{N\Delta^2}.
\end{align*}
\end{lemma}

Moreover, we need following inequalities on the confidence bound of the average rewards. Similar results have also been proved in \cite{menard2017minimax} for bounding the KL divergence between two exponential family distributions for different arms.
\begin{lemma}\label{lem:colt17}
Let $\delta>0$  and $M_1, M_2,\ldots,M_n$ be 1-subgaussian random variables with zero means. Denote $\hat{\mu}_n=\sum_{s=1}^n M_s/n$. Then the following statements hold:
\begin{enumerate}
    \item  for any $T_1 \leq T$, 
    \begin{equation}
        \sum_{n=1}^T \mathbb{P}\bigg(\hat{\mu}_{n}+\sqrt{\frac{4}{n}\log^{+}\bigg(\frac{T_1}{n}\bigg)} \geq \delta \bigg)\leq 1+\frac{4\log^{+}({T_1}{\delta^2})}{\delta^2} +\frac{3}{\delta^2}+\frac{\sqrt{8\pi {\log^{+}({T_1}{\delta^2})}}}{\delta^2};
    \end{equation}
    
    \item if $T\delta^2\geq e^2$, then 
    \begin{align}
    \sum_{n=1}^T \mathbb{P}\Bigg(\hat{\mu}_{n}+\sqrt{\frac{2}{n}\log\bigg(\frac{T}{n}\bigg(\log^2\frac{T}{n}+1\bigg)\bigg)} \geq \delta \Bigg)
    &\leq 1+\frac{2\log({T}{\delta^2}(\log^2({T}{\delta^2})+1))}{\delta^2}+\frac{3}{\delta^2}\notag\\&\qquad+\frac{\sqrt{4\pi {\log({T}{\delta^2}(\log^2({T}{\delta^2})+1))}}}{\delta^2};
    \end{align}
    \item if $T\delta^2\geq 4e^3$, then 
    \begin{equation}
    \label{lem:main-equation-3}
        \mathbb{P}\bigg(\exists s\leq T: \hat{\mu}_{s}+\sqrt{\frac{2}{s}\log\bigg(\frac{T}{s}\bigg(\log^2\frac{T}{s}+1\bigg)\bigg)}+\delta\leq 0 \bigg)\leq \frac{4(16e^2+1)}{T\delta^2}.
    \end{equation}
\end{enumerate}
\end{lemma}

\begin{proof}[Proof of Theorem \ref{theorem:unknowndelta}]
Let $\tau_1$ be the number of times each arm is pulled in \emph{Stage I} of Algorithm~\ref{alg:double-exploration-unknown-gap} and $\tau_2$ be  the total number of times arm $2'$ is pulled in \emph{Stage III} of Algorithm~\ref{alg:double-exploration-unknown-gap}. 
Similar to \eqref{eq:alg1gre}, the regret of Algorithm~\ref{alg:double-exploration-unknown-gap} can be decomposed as follows
\begin{align}\label{eq:unknowndeltau_regret_decomp}
    R_{\mu}(T)
    &\leq\underbrace{\Delta T_1 \mathbb{P}(\tau_1<T,1'=2)}_{I_1}
  +\underbrace{\Delta \EE[\tau_1]+\Delta\mathbb{E}[{\tau_2}]}_{I_2}+ \underbrace{\Delta T\mathbb{P}(\tau_2< T,a=2)}_{I_3}.  
\end{align}
 
Since we focus on the asymptotic optimality, we define $\epsilon_T=\sqrt{2\log (T\Delta^2)/(T_1\Delta^2)}$ and assume $\epsilon_T\in (0,1/2)$, $T\Delta^2\geq 16e^3$.\\
\noindent\textbf{Bounding term $I_1$:} Let $X_s$ and $Y_s$ be the reward of arm $1$ and $2$ when they are pulled for the $s$-th time respectively, $s=1,2,\ldots$.  Recall that $\hmu_{k,s}$ is the average reward for arm $k$ after its $s$-th pull.    
Applying Lemma~\ref{lem:peeling-simp}, we have
\begin{align}
\label{eq:unknowndeltau_stage1_bound}
    \PP(\tau_1<T,1'=2)& \leq \PP\bigg( \exists s\in\NN: 2s\leq T, \ \hat{\mu}_{1,s}-\hat{\mu}_{2,s}\leq -\sqrt{\frac{8\log^{+}(T_1/(2s))}{s}} \bigg) \notag\\
    & \leq \frac{30}{T_1\Delta^2}.
\end{align}
where the last inequality comes from Lemma~\ref{lem:peeling-simp}. Therefore $I_1\leq 30/\Delta$.  


\noindent \textbf{Bounding term  $I_2$:} 
By the definition of $\tau_1$ and the stopping rule of $\stageOne$ in Algorithm \ref{alg:double-exploration-unknown-gap}, we have
\begin{align}\label{eq:unknowndeltau_1}
    \mathbb{E}[\tau_1]=\sum_{s=1}^{T}\PP(\tau_1\geq s) & \leq \sum_{s=1}^{T/2}\PP \bigg(\hat{\mu}_{1,s}-\hat{\mu}_{2,s}\leq \sqrt{\frac{8\log^{+}({T_1}/({2s}))}{s}} \bigg) \notag\\
    & = \sum_{s=1}^{T/2} \PP \bigg(\frac{\sum_{i=1}^s Z_i}{s}\leq \sqrt{\frac{4}{s}\log^{+} \Big(\frac{T_1}{2s}\Big)}-\frac{\Delta}{\sqrt{2}} \bigg) \notag\\
    & \leq  \sum_{s=1}^{T} \PP \bigg(-\frac{\sum_{i=1}^s Z_i}{s}+\sqrt{\frac{4}{s}\log^{+} \Big(\frac{T_1/2}{s}\Big)}\geq \frac{\Delta}{\sqrt{2}} \bigg) \notag\\
   & \leq  1+\frac{8\log^{+}(T_1\Delta^2/4)}{\Delta^2}+\frac{6}{\Delta^2}+\frac{2\sqrt{8\pi\log^{+}(T_1\Delta^2/4)}}{\Delta^2},
\end{align}
where the equality is by the definition of  $\sum_{i=1}^s Z_i/s=\sum_{i=1}^s(X_i-Y_i-\Delta)/(\sqrt{2}s)=(\hmu_{1,s}-\hmu_{2,s}-\Delta)/\sqrt{2}$, and the last inequality is due to the first statement of Lemma~\ref{lem:colt17} since $-Z_i$ are 1-subgaussian variables as well.

Let 
\begin{align}
\label{eq:def:eps_T}
    \epsilon_T=\sqrt{2\log (T\Delta^2)/(T_1\Delta^2)}.
\end{align}
 Since we focus on the asymptotic optimality ($T\rightarrow \infty$) and $T_1=\log^2 T$, we assume 
\begin{align}
\label{eq:unknown-gap-assumption}
   \epsilon_T\in (0,1/2) \qquad {\text and } \qquad T\Delta^2\geq 16e^3. 
\end{align}
Let $E$ be the event $\mu'\in [\mu_{1'}-\epsilon_T\Delta,\mu_{1'}+\epsilon_T\Delta]$.
Applying Lemma~\ref{lem:subguassian} and union bound, $\PP(E)\geq 1-2/(T\Delta^2)$.   Similar to~\eqref{eq:bound-I_2-added}, we have
\begin{align}
   \label{eq:bound-I_2-added-unknown}
    \EE[\tau_2]\leq\EE[\tau_2\ind(E, 1'=1)]+\EE[\tau_2 \ind(E, 1'=2)]+2/\Delta^2.
\end{align}
To bound $\EE[\tau_2\ind(E, 1'=1)]$, we assume  event $E$ holds and the chosen arm $1'$ is the best arm, i.e., $1'=1$. Let $E_1=\{E, 1'=1\}$.  Let $\Delta'=\mu'-\mathbb{E}[Y_{i+\tau_1}]$. Then conditioned on $E_1$, $\Delta'\in [(1-\epsilon_T)\Delta, (1+\epsilon_T)\Delta]$. Since $\epsilon_T \in (0,1/2)$ and $T\Delta^2\geq 16 e^3$,  we have that conditioned on $E_1$, $T(\Delta')^2\geq(1-\epsilon_T)^2T\Delta^2\geq 4e^3$. 
Let $W_i={\mu'}-Y_{i+\tau_1}-\Delta'$. Then $-W_i$ is 1-subgaussian random variable. By the stopping rule of $\stageThree$ in Algorithm \ref{alg:double-exploration-unknown-gap}, it holds that 
\begin{align}\label{eq:expectedtau_2}
        \mathbb{E}[\tau_2\mid E_1]&\leq\sum_{t_2=1}^{T} \PP(\tau_2\geq t_2\mid E_1) \notag\\
        & = \sum_{t_2=1}^{T}\PP\bigg( \mu'-\theta_{2',t_2}\leq \sqrt{\frac{2}{t_2}\log\Big(\frac{T}{t_2}\Big(\log^2\frac{T}{t_2}+1\Big)\Big) } \;\bigg|\; E_1\bigg) \notag\\
         & =\sum_{t_2=1}^{T}\PP\bigg( -\frac{\sum_{i=1}^{t_2}W_i}{t_2}+\sqrt{\frac{2}{t_2}\log\Big(\frac{T}{t_2}\Big(\log^2\frac{T}{t_2}+1\Big)\Big) }\geq \Delta' \;\bigg|\; E_1  \bigg) \notag\\
           & \leq 1+\frac{3+2\log(4T\Delta^2(\log^2(4T\Delta^2)+1)) +\sqrt{4\pi\log(4T\Delta^2(\log^2(4T\Delta^2)+1))}}{(1-\epsilon_T)^2\Delta^2}.
\end{align}
where the last inequality is due to  the second statement of Lemma~\ref{lem:colt17} and $-W_i$ are $1$-subGuassian. Let $E_2=\{E,1'=2\}$, using the same argument, we can derive same bound as in \eqref{eq:expectedtau_2} for $\mathbb{E}[\tau_2\mid E_2]$. Then We have
\begin{align}
\label{eq:unknone-bound-I_2-added}
    \Delta\EE[\tau_2]& \leq\Delta\EE[\tau_2\ind(E_1)]+\Delta\EE[\tau_2 \ind(E_2)]+\frac{2}{\Delta} \notag \\
    & \leq \Delta+\frac{2}{\Delta}  +\frac{3+2\log(4T\Delta^2(\log^2(4T\Delta^2)+1)) +\sqrt{4\pi\log(4T\Delta^2(\log^2(4T\Delta^2)+1))}}{(1-\epsilon_T)^2\Delta}.
  \end{align}

\noindent\textbf{Bounding term $I_3$:}
$\mathbb{P}(\tau_2< T,a=2)$ is the joint probability between the event that the chosen arm after $\stageThree$ is arm $2$ and the event that the following stopping condition will be satisfied in $\stageThree$:
\begin{align}
   |\mu'-\theta_{2',t_2}|<\sqrt{2/t_2\log\big(T/t_2\big(\log^2(T/t_2)+1\big)\big)}.
\end{align}
Similar to~\eqref{eq:bound-I_2-added-unknown},  
\begin{align}
\label{eq:unknown-added-1}
     I_3\leq\Delta T\PP[\tau_2< T,a=2 \mid E_1]\PP[E_1]+\Delta T\PP[\tau_2< T,a=2 \mid E_2]\PP[E_2]+\frac{2}{\Delta}.
\end{align}
Again, we first assume $E_1$ holds. By definition, we have that conditioned on $E_1$, $\sum_{i}^sW_i/s=\mu'-\theta_{2',s}-\Delta'$ and $W_i$ is 1-subgaussian with zero mean. Recall that we have $T(\Delta')^2\geq 4e^3$. By the third statement of Lemma~\ref{lem:colt17}, we have
\begin{align}\label{eq:2arm_unknowngap_term3}
 \PP(\tau_2<T,a=2 \mid E_1)   
&\leq  \PP\bigg(\exists t_2\geq 1, \mu'-\theta_{2',t_2}+\sqrt{\frac{2}{t_2}\log\Big(\frac{T}{t_2}\Big(\log^2\frac{T}{t_2}+1\Big)\Big) }  \leq 0 \;\bigg|\; E_1 \bigg) \notag\\
&\leq  \PP\bigg(\exists t_2\geq 1, \mu'-\theta_{2',t_2}-\Delta'+\Delta'+\sqrt{\frac{2}{t_2}\log\Big(\frac{T}{t_2}\Big(\log^2\frac{T}{t_2}+1\Big)\Big) }  \leq 0  \;\bigg|\; E_1 \bigg) \notag\\
 &\leq   \frac{4(16e^2+1)}{T(1-\epsilon_T)^2\Delta^2}. 
\end{align}
When $E_2$ holds, the proof is similar to the previous one. In particular, we only need to change the notations to $\Delta'=\EE[X_{i+\tau_1}]-\mu'$, which  satisfies conditioned on $E_2$, $\Delta'\in [(1-\epsilon_T)\Delta, (1+\epsilon_T)\Delta]$. Hence, we can derive same bound as \eqref{eq:2arm_unknowngap_term3} for term $\PP(\tau_2<T,a=2 \mid E_2)$ .   \\
Therefore, 
\begin{align}
\label{eq:unknone-bound-I_3}
    I_3= \Delta T\mathbb{P}(\tau_2< T,a=2)\leq \frac{2}{\Delta}+ \frac{4(16e^2+1)}{(1-\epsilon_T)^2\Delta}.
\end{align}

\noindent\textbf{Completing the proof:}
  Therefore, substituting \eqref{eq:unknowndeltau_stage1_bound},  \eqref{eq:unknowndeltau_1}, \eqref{eq:unknone-bound-I_2-added} and \eqref{eq:unknone-bound-I_3}  into \eqref{eq:unknowndeltau_regret_decomp}, we have
\begin{align}
\label{eq:unknowndelta-final}
    R_{\mu}(T)& \leq 2\Delta+ \frac{40+8\log^{+}(T_1\Delta^2/4)+2\sqrt{8\pi\log^{+}(T_1\Delta^2)}}{\Delta}  \\ 
    &\qquad +\frac{4(16e^2+2)+2\log(4T\Delta^2(\log^2(4T\Delta^2)+1)) +\sqrt{4\pi\log(4T\Delta^2(\log^2(4T\Delta^2)+1))}}{(1-\epsilon_T)^2\Delta}.
\end{align}
Recall that $\epsilon_T^2={2\log(T\Delta^2)}/({T_1\Delta^2})$.
Let $T_1=\log^2 T$. When $T\rightarrow \infty$, we have $\epsilon_T\rightarrow 0$, and hence $\lim_{T\rightarrow \infty}R_{\mu}(T)/T =2/\Delta$.
\end{proof}

\section{Proof of the Regret Bound of Algorithm~\ref{alg:double-exploration-unknown-gap-minimax}}\label{sec:proof_of_two_arm_all_optimal}
In this section, we provide the proof of Theorem~\ref{theorem:unknowndelta-minimax}. It will mostly follow the proof framework in Section \ref{sec:proof_of_thm_unknowdelta_2arm} for Theorem \ref{theorem:unknowndelta}. Recall that in the proof of Theorem \ref{theorem:unknowndelta}, we used the concentration inequalities in Lemma~\ref{lem:colt17} to upper bound $\tau_2$, which is the total number of times that the suboptimal arm $2'$ is pulled in $\stageThree$ of Algorithm \ref{alg:double-exploration-unknown-gap}. Now in Line~\ref{line-alg3-4} of Algorithm~\ref{alg:double-exploration-unknown-gap-minimax}, we added the extra stopping time $\log^2 T$ to $\stageThree$. Therefore, Lemma~\ref{lem:colt17} is no longer directly applicable here. Instead, we need the following refined concentration lemma.  
\begin{lemma}\label{lem:colt17-1}
Let $\delta\in(0,2/\log^4 T)$ and $M_1, M_2,\ldots,M_n$ be 1-subgaussian random variables with zero means. Denote $\hat{\mu}_n=\sum_{s=1}^n M_s/n$. Then the following inequality holds:
   \begin{equation}
        \mathbb{P}\bigg(\exists s\leq \log^2 T: \hat{\mu}_{s}+\sqrt{\frac{2}{s}\log\bigg(\frac{eT}{s}\bigg(\log^2\frac{T}{s}+1\bigg)\bigg)}-\delta\leq 0 \bigg)\leq \frac{16e^2\log T}{T}.
    \end{equation}
\end{lemma}

\begin{proof}[Proof of Theorem~\ref{theorem:unknowndelta-minimax}]
For the sake of simplicity, we use the same notation that used in Theorem~\ref{theorem:unknowndelta}. 
Similar to \eqref{eq:alg1gre}, the regret of Algorithm~\ref{alg:double-exploration-unknown-gap-minimax} can be decomposed as follows
\begin{align}\label{eq:unknowndeltau_regret_decomp-minimax}
    R_{\mu}(T)
    &\leq\underbrace{\Delta T_1 \mathbb{P}(\tau_1<T_1,1'=2)}_{I_1}
  +\underbrace{\Delta \EE[\tau_1]+\Delta\mathbb{E}[{\tau_2}]}_{I_2}+ \underbrace{\Delta T\mathbb{P}(\tau_2<\log^2 T,a=2)}_{I_3}\notag \\
  & \qquad +\underbrace{\PP(\tau_2=\log^2 T)R(IV\mid \tau_2=\log^2 T)}_{I_4},  
\end{align}
where terms $I_1$, $I_2$ and $I_3$ are  the same as or slightly different from that in \eqref{eq:alg1gre}, and term $I_4$ is a new regret caused by Lines \ref{line-pis}-\ref{line-pis-1} of Algorithm \ref{alg:double-exploration-unknown-gap-minimax}, where $\tau_2=\log^2 T$ and $R(IV\mid \tau_2=\log^2 T)$ represents the regret of Lines \ref{line-pis}-\ref{line-pis-1} in $\stageFour$.


\noindent{\bf Proof of Asymptotic Optimality:} The proof of the asymptotic optimality is almost the same as that in Section \ref{sec:proof_of_thm_unknowdelta_2arm}. Recall the definition in \eqref{eq:def:eps_T} that $\epsilon_T=\sqrt{2\log (T\Delta^2)/(T_1\Delta^2)}$. To derive the asymptotic regret bound, since we consider the case that $T\rightarrow \infty$, we can trivially assume $\epsilon_T\in(0,1/2)$ and $T\Delta^2\geq 16e^3$. Note that $\stageOne$ and $\stageTwo$ of Algorithm \ref{alg:double-exploration-unknown-gap-minimax} are exactly the same as that of Algorithm \ref{alg:double-exploration-unknown-gap}. Based on the proof in Section \ref{sec:proof_of_thm_unknowdelta_2arm}, it is easy to obtain the following results. 
\begin{align}
&\Delta T_1 \mathbb{P}(\tau_1<T_1,1'=2)=O\bigg(\frac{1}{\Delta}\bigg),\label{eq:minimax-c.3}\\
   &\Delta\EE[\tau_1]=O\bigg(\Delta+\frac{\log^+(T_1\Delta^2)}{\Delta}\bigg),\label{eq:minimax-c.4}\\
   &\Delta\EE[\tau_2]\leq\Delta+ \frac{O(1)+2\log(4e\cdot T\Delta^2(\log^2(4e\cdot T\Delta^2)+1))}{(1-\epsilon_T)^2\Delta}\notag\\
   &\qquad\qquad+\frac{\sqrt{4\pi\log(4e\cdot T\Delta^2(\log^2(4e\cdot T\Delta^2)+1))}}{(1-\epsilon_T)^2\Delta}. \label{eq:minimax-c.5}
\end{align}
which are due to \eqref{eq:unknowndeltau_stage1_bound}, \eqref{eq:unknowndeltau_1} and \eqref{eq:unknone-bound-I_2-added} respectively. 

For term $I_3$, $\mathbb{P}(\tau_2< \log^2 T,a=2)$ is the joint probability between the event that the chosen arm after $\stageThree$ is the suboptimal arm $2$ and the event that the following stopping condition will be satisfied after less than $\log^2 T$ time steps executed in $\stageThree$:
\begin{align}
   |\mu'-\theta_{2',t_2}|<\sqrt{2/t_2\log\big(eT/t_2\big(\log^2(T/t_2)+1\big)\big)}.
\end{align}
Recall the proof in Section \ref{sec:proof_of_thm_unknowdelta_2arm} and note that the above probability is smaller than that in Algorithm \ref{alg:double-exploration-unknown-gap} due to the extra requirement $\tau_2<\log^2 T$.  Therefore, by \eqref{eq:unknone-bound-I_3} we have
\begin{align}
\label{eq:minimax-c.7}
   I_3= \Delta T\mathbb{P}(\tau_2<\log^2 T,a=2)= O\bigg(\frac{1}{(1-\epsilon_T)^2\Delta}\bigg).
\end{align}
 Now, we bound the new term $I_4$. Note that $\tau_2=\log^2 T$ implies  Lines \ref{line-pis}-\ref{line-pis-1}  is performed in $\stageFour$. Let $\tau_3$ be the number of pulls of each arm in Line \ref{line-alg3-18}. Then the regret in Lines \ref{line-pis}-\ref{line-pis-1} can be upper bounded as $R(IV\mid\tau_2=\log^2 T)\leq \Delta\EE[\tau_3]+\Delta T\PP(\tau_3<T, a=2)$.  Similar to the proof in \eqref{eq:unknowndeltau_stage1_bound}, we have
 \begin{align}
 \label{eq:minimax-asym-2}
 \PP(\tau_3<T, a=2) &\leq \PP\bigg( \exists s\in\NN: 2s\leq T, \ p_{1,s}-p_{2,s}\leq -\sqrt{\frac{8\log(T/s)}{s}} \bigg) \leq \frac{15}{T\Delta^2}.
 \end{align}
Similar to the proof of \eqref{eq:unknowndeltau_1}, we have
 \begin{align}
 \label{eq:minimax-asym-3}
     \EE[\tau_3]=\sum_{s=1}^{T}\PP(\tau_3\geq s)&\leq \sum_{s=1}^{T}\PP\bigg(p_{1s}-p_{2s}\leq \sqrt{\frac{8\log(T/s)}{s}} \bigg) \notag \\
     &\leq 1+\frac{8\log(T\Delta^2)}{\Delta^2}+\frac{6}{\Delta^2}+\frac{2\sqrt{8\pi\log(T\Delta^2)}}{\Delta^2}.
 \end{align}
 Therefore, the regret generated by Lines \ref{line-pis}-\ref{line-pis-1} is
 \begin{align}
  \label{eq:mini-add}
         R(IV\mid \tau_2=\log^2T)=O\bigg(\Delta+\frac{
        \log (T\Delta^2)}{\Delta}\bigg).
 \end{align}
To obtain the final bound for term $I_4$, we need to calculate the probability $\PP(\tau_2=\log^2 T)$. Since
\begin{align}
   \EE[\tau_2]&=\EE[\tau_2|\tau_2=\log^2 T]\PP(\tau_2=\log^2 T)+\EE[\tau_2|\tau_2<\log^2 T]\PP(\tau_2<\log^2 T) \notag \\
   &\geq \log^2 T\PP(\tau_2=\log^2 T),
\end{align}
combining the above result with \eqref{eq:minimax-c.5}, we have 
 \begin{align}
 \label{eq:minimax-asy-4}
   \PP(\tau_2=\log^2 T)= O\bigg(\frac{\log(T\Delta^2)}{\Delta^2\log^2 T}\bigg).
 \end{align}
Combining \eqref{eq:mini-add} and \eqref{eq:minimax-asy-4} together, we have
 \begin{align*}
      \lim_{T\rightarrow \infty}\frac{I_4}{\log T}=\lim_{T\rightarrow \infty}\frac{\PP(\tau_2=\log^2 T)R(IV\mid \tau_2=\log^2 T)}{\log T}=0.
 \end{align*}
In conclusion, substituting the above results back into the regret decomposition in \eqref{eq:unknowndeltau_regret_decomp-minimax}, we have $\lim_{T\rightarrow \infty}R_{\mu}(T)/\log T=2/\Delta$, which proves the asymptotic optimality of Algorithm \ref{alg:double-exploration-unknown-gap-minimax}.

\noindent{\bf Proof of Minimax/Instance-Dependent Optimality:} 
When $T\Delta^2 \leq 16e^3$, we have $R_{\mu}(T)\leq T\Delta=O(\sqrt{T})$ and $R_{\mu}(T)\leq T\Delta= O(1/\Delta)$, which is trivially minimax/instance-dependant optimal. Hence,  we  assume $T\Delta^2 \geq 16e^3$ in the rest of the proof. Different from the previous proof, $\epsilon_T$ defined in \eqref{eq:def:eps_T} may not fall in the interval $(0,1/2)$ now. In particular, when the gap $\Delta$ is very small, the estimation of $\mu_{1'}$ will not be sufficiently accurate such that $\mu'\in[\mu_{1'}-\epsilon_T\Delta, \mu_{1'}+\epsilon_T\Delta]$. To handle this scenario, we will consider the following two cases. \\
{\bf Case 1: $\Delta>1/\log^4 T$.} Actually, if the unknown gap $\Delta$ is larger than $1/\log^4 T$, the proofs in the previous part for the asymptotic optimality still holds. Note that $T_1=\log^{10} T$, then $\epsilon_T=\sqrt{2\log(T\Delta^2)/T_1\Delta^2}\in(0,1/2)$. By the same argument  as in  \eqref{eq:minimax-c.3}, \eqref{eq:minimax-c.4}, \eqref{eq:minimax-c.5} and \eqref{eq:minimax-c.7}, we have $I_1+I_2+I_3=O(\Delta+\log(T\Delta^2)/\Delta)$. Also by \eqref{eq:mini-add}, we  have $I_4\leq R(IV|\tau_2=\log^2 T)=O(\Delta+{\log(T\Delta^2)}/{\Delta})$. Thus substituting these terms back into the regret decomposition in \eqref{eq:unknowndeltau_regret_decomp-minimax} yields $R_{\mu}(T)=O(\Delta+\log(T\Delta^2)/\Delta)=O(\Delta+\sqrt{T})$. \\
{\bf Case 2: $\Delta<1/\log^4 T$.} In this case, term $I_1$ and $\Delta\EE[\tau_1]$ can be still bounded in the same way  as  in   \eqref{eq:minimax-c.3}, \eqref{eq:minimax-c.4}, 
which leads to $I_1+\Delta\EE[\tau_1]=O({1}/{\Delta}+{\log^+(T_1\Delta^2)}/{\Delta})$. 

Now we bound terms $\EE[\tau_2]$ and $I_3$.  Recall that in the previous part for proving the asymptotic regret,  the bounds of term $\EE[\tau_2]$ in \eqref{eq:minimax-c.5} and term $I_3$ in \eqref{eq:minimax-c.7} are heavily based on the results in  \eqref{eq:unknone-bound-I_2-added} and \eqref{eq:unknone-bound-I_3}. However, the results in \eqref{eq:unknone-bound-I_2-added} and  \eqref{eq:unknone-bound-I_3} only hold based on the assumption $\epsilon_T\in (0,1/2)$, which is not true in the case $\Delta<1/\log^4 T$. Hence, \eqref{eq:minimax-c.5} and \eqref{eq:minimax-c.7} are not applicable here. Now, we bound these terms without assuming  $\epsilon_T\in(0,1/2)$.  For term $\Delta\EE[\tau_2]$, since we pull $2'$ at most $\log^2 T$ times in Stage {\it III} of Algorithm \ref{alg:double-exploration-unknown-gap-minimax}, it can be trivially seen that $\Delta\EE[\tau_2]\leq \Delta\log^2 T\leq 1$. 

For term $I_3$, note that we have pulled arm $1'$ for $T_1=\log^{10} T$ times after $\stageTwo$. Applying Lemma~\ref{lem:subguassian}, we obtain
\begin{align*}
    \PP(|\mu'-\mu_{1'}|\geq 1/\log^4 T)\leq  2/T^{1/2\log T}\leq 1/T, 
\end{align*}
where $\mu'$ is the average reward for arm $1'$ at the end of $\stageTwo$ and we used the fact that $ T\geq e^3$. Define event $E'=\{|\mu'-\mu_{1'}|\leq 1/\log^4 T\}$ and its complement as $E'^{c}$.  We further have
\begin{align}\label{eq:proof_alg_minimax_I3_decomp}
    I_3&\leq \Delta T\PP(\tau_2<\log^2 T, a=2 \mid E')+\Delta T\PP(E'^c) \notag  \\
    &  \leq \Delta T\PP(\tau_2<\log^2 T  \mid E')+\Delta. 
\end{align}
Conditioned on event $E'$, we have  $|\mu'-\mu_{2'}|\leq|\mu'-\mu_{1'}|+|\mu_{1'}-\mu_{2'}|\leq 2/\log^4 T$ since $|\mu_{1'}-\mu_{2'}|=\Delta<1/\log^4 T$.  Based on this observation, we have 
\begin{align*}
    & \PP(\tau_2<\log^2 T\mid E') \notag \\
    & \leq \PP\bigg( \exists t_2\leq \log^2 T, |\mu'-\theta_{2',t_2}|\geq   \sqrt{\frac{2}{t_2}\log \bigg(\frac{eT}{t_2}\bigg( \log^2 \frac{T}{t_2}+1\bigg) \bigg)} \ \bigg| \ E' \bigg) \notag \\
    &\leq \PP\bigg(\exists t_2\leq \log^2 T, -(\mu'-\theta_{2',t_2})+\sqrt{\frac{2}{t_2}\log \bigg(\frac{eT}{t_2}\bigg( \log^2 \frac{T}{t_2}+1\bigg) \bigg)}\leq0 \ \bigg| \ E'\bigg)\notag \\
    &\qquad + \PP\bigg(\exists t_2\leq \log^2 T, -(\mu'-\theta_{2',t_2})-\sqrt{\frac{2}{t_2}\log \bigg(\frac{eT}{t_2}\bigg( \log^2 \frac{T}{t_2}+1\bigg) \bigg)}\geq0 \ \bigg| \ E'\bigg) \notag \\
    &\leq \PP\bigg(\exists t_2\leq \log^2 T, (\mu'-\mu_{2'})-(\mu'-\theta_{2',t_2})-|\mu'-\mu_{2'}|+\sqrt{\frac{2}{t_2}\log \bigg(\frac{eT}{t_2}\bigg( \log^2 \frac{T}{t_2}+1\bigg) \bigg)}\leq0 \ \bigg| \ E'\bigg)\notag \\
    &\qquad + \PP\bigg(\exists t_2\leq \log^2 T, (\mu'-\mu_{2'})-(\mu'-\theta_{2',t_2})+|\mu'-\mu_{2'}|-\sqrt{\frac{2}{t_2}\log \bigg(\frac{eT}{t_2}\bigg( \log^2 \frac{T}{t_2}+1\bigg) \bigg)}\geq0 \ \bigg| \ E'\bigg) \notag \\
    &\leq \frac{32e^2\log T}{T},
\end{align*}
where the first inequality is due to the stopping rule of $\stageThree$ in Algorithm \ref{alg:double-exploration-unknown-gap-minimax}, the second inequality is due to the fact that $\{|x-y|\geq z\}\subset\{x-y\geq z\}\bigcup\{x-y\leq -z\}$, the third inequality is due to the fact that $\mu'-\mu_{2'}-|\mu'-\mu_{2'}|\leq 0$ and   $\mu'-\mu_{2'}+|\mu'-\mu_{2'}|\geq 0$,
and in the last inequality we apply Lemma \ref{lem:colt17-1} with $\delta=|\mu'-\mu_{2'}|$. 
Therefore, substituting the above inequality back into \eqref{eq:proof_alg_minimax_I3_decomp},  we have the following bound for term $I_3$:
\begin{align*}
   I_3&\leq \Delta T\PP(\tau_2<\log^2 T| E')+\Delta \leq \Delta T \frac{16e^2\log T}{T}+\Delta\leq 1,
\end{align*}
where the last inequality is due to $\Delta<1/\log^4 T$ and the fact that $T\geq e^3$.
For term $I_4$, conditioned on $\tau_2=\log^2 T$, the regret of $\stageFour$ (namely, term $R(IV \mid \tau_2=\log^2 T)$) only depends on the data collected in Lines \ref{line-pis}-\ref{line-pis-1} of Algorithm \ref{alg:double-exploration-unknown-gap-minimax}, which is therefore the same as in \eqref{eq:mini-add}. we have 
\begin{align*}
    I_4\leq R(IV \mid \tau_2=\log^2 T)=O\bigg(\Delta+\frac{\log(T\Delta^2)}{\Delta}\bigg).
\end{align*} 
Hence, for case 2, the total regret $R_{\mu}(T)=O(\Delta+\log(T\Delta^2)/\Delta)=O(\Delta+\sqrt{T})$.
\end{proof}

\section{Proof of the Regret Bound of Algorithm \ref{alg:double-exploration-unknown-gap-K}}\label{sec:proof_of_thm_k_arm}
In this section, we prove the regret bound of DETC for $K$-armed bandits.
\begin{proof}[Proof of Theorem \ref{theorem:unknowndelta_Karm}]
Let $T_i$ be the total number of pulls of arm $i$ throughout the algorithm, $i\geq 2$.  Since by definition the regret is $R_{\mu}(T)=\sum_{i}\EE[T_i\Delta_i]$, it suffices to prove 
\begin{align}
    \lim_{T\rightarrow \infty} \frac{\EE[T_i]}{\log (T)}=\frac{2}{\Delta_i^2}.
\end{align}
Denote $\tau_{2,i}$ as the  number of pulls of arm $i$ in {\it Stage III} of Algorithm \ref{alg:double-exploration-unknown-gap-K}. Similar to \eqref{eq:no_pull_arm_decomp} and \eqref{eq:alg1gre}, the term $\EE[T_i]$ can be decomposed as follows
\begin{align}\label{eq:unknowndeltau_regret_decomp-top-k}
    \EE[T_i]
    &\leq \sqrt{\log T}+\underbrace{ \log^2 T\mathbb{P}(1'=i)}_{I_1}
   +\underbrace{\mathbb{E}[{\tau_{2,i}}]}_{I_2}+ \underbrace{T\mathbb{P}(\hat{\mu}_{1'}\geq\theta_{j',t_j}, \failflag=0,a=i)}_{I_3}\notag \\
  & +\underbrace{ \log^2 T \mathbb{P}(\failflag=1)+T\mathbb{P}(\failflag=1,a=i)}_{I_4},  
\end{align}
where the last term $I_4$ characterizes the failing probability of the first three stages and the ETC step in the last two lines of Algorithm \ref{alg:double-exploration-unknown-gap-K}.

\noindent\textbf{Bounding term $I_1$:}
 Let $\hat{\mu}_{i,s}$ be the estimated reward of arm $i$ after its $s$-th pull. Let $\tau_1=\sqrt{\log T}$. Let $X$ be the reward of arm $1$ and $Y^i$ be the reward of arm $i$ for $i>1$. Let $S^i_{n}=X_1-Y^i_{1}+\cdots +X_n-Y^i_{n}$. After pulling arm $1$ and arm $i$ $\tau_1$ times, using Lemma~\ref{lem:subguassian},  we get
\begin{align}
\label{eq:k-I_1-1}
    \PP(S^i_{\tau_1}/\tau_1\leq \Delta_i-\epsilon)\leq \exp(-\tau_1\epsilon^2/4).
\end{align}
For sufficiently large $T$ such that $T>K$ and for all $i$, it holds
\begin{equation}
\label{kassume1}
    \frac{\sqrt{\log T}}{\log K+2\log \log T}\geq\frac{4}{\Delta_i^2},
\end{equation}
Setting $\epsilon=\Delta_i$ in~\eqref{eq:k-I_1-1}, we have $\PP(\hat{\mu}_{1,\tau_1}\leq \hat{\mu}_{i,\tau_1})\leq 1/( K\log^2 T)$. Applying union bound, we have 
\begin{align}
\label{eq:karm-e5}
    \PP(\hat{\mu}_{1,\tau_1}\geq \max_{i} \hat{\mu}_{i,\tau_1})=\PP(1'=1) \geq 1-\frac{1}{\log^2 T},
\end{align}
which further implies  $I_1\leq 1$.

\noindent\textbf{Bounding term $I_2$:}
 Let $\epsilon_i=\sqrt{4\log (T\Delta_i^2)/((\log T)^2\Delta_i^2)}$.  Applying  Lemma~\ref{lem:subguassian}, we have 
\begin{align}
\label{eq:top-k-mu'-bound}
    \PP(\mu'\notin (\mu_{1'}-\epsilon_i\Delta_i,\mu_{1'}+\epsilon_i\Delta_i))\leq 2/(T\Delta_i^2).
\end{align} Similar to~\eqref{eq:assume2}, we choose a large $T$ such that for all $\Delta_i>0$,
\begin{align}
\label{kassume2}
    \sqrt{\frac{4\log(T\Delta_i^2)}{\Delta_i^2\log^2 T}}\leq \frac{1}{(\log T)^{\frac{1}{3}}},
\end{align}
then $\epsilon_i\leq 1/(\log T)^{\frac{1}{3}}$.
Let $E$ be the event $\mu'\in (\mu_{1'}-\epsilon_i\Delta_i,\mu_{1'}+\epsilon_i\Delta_i)$.  Let $E_1$ be the event $\{E,1'=1\}$. Note that $\Pr(1'=1)\geq 1-1/\log^2 T$, $\Pr(E^c)\leq 2/(T\Delta_i^2)$ and $\tau_{2,i}\leq \log^2 T$, the term $I_2$ can be decomposed as 
\begin{align}
\label{eq:top-k-bound-Etau2i}
    \EE[\tau_{2,i}]& =\EE[\tau_{2,i} \ind(1'=1)]+\EE[\tau_{2,i} \ind(1'\neq 1)] \notag \\
    & \leq \EE[\tau_{2,i} \ind(1'=1)]+1 \notag \\
      & \leq \EE[\tau_{2,i} \ind(E_1)]+\EE[\tau_{2,i} \ind(E^c)]+1 \notag \\
    & \leq  1+\frac{2}{\Delta_i}+\EE[\tau_{2,i} \mid E_1] .
\end{align}
We can derive the same bound as $\EE[\tau_2 \mid E_1]$ in~\eqref{eq:expectedtau_2} for  $\EE[\tau_{2,i} \mid E_1]$. We have
\begin{align}\label{eq:expectedtau_2-top-k}
        I_2&=\mathbb{E}[\tau_{2,i}\mid E_1]\notag\\
        &\leq 1+\frac{3+2\log(4T\Delta_i^2(\log^2(4T\Delta_i^2)+1)) +\sqrt{4\pi\log(4T\Delta_i^2(\log^2(4T\Delta_i^2)+1))}}{(1-\epsilon_i)^2\Delta_i^2}.
\end{align}

\noindent\textbf{Bounding term $I_3$:}
When $\failflag=0$, we can follow the same proof for bounding $I_3$ in~\eqref{eq:2arm_unknowngap_term3}. Therefore, we can obtain
\begin{align}
\label{eq:unknone-bound-I_3-top-k}
    I_3\leq \frac{2}{\Delta_i^2}+ \frac{4(16e^2+1)}{(1-\epsilon_i)^2\Delta_i^2}.
\end{align}

\noindent\textbf{Bounding term $I_4$:}
For term $\PP(\failflag=1)$, similar to~\eqref{eq:top-k-bound-Etau2i}, we have
\begin{align}\label{eq:fail_eq_1}
    \PP(\failflag=1) &= \PP(\failflag=1\mid 1'=1) \Pr(1'=1)+\PP(\failflag=1\mid 1'\neq 1)\Pr(1'\neq 1) \notag \\
    & \leq  \PP(\failflag=1\mid 1'=1)+\frac{1}{\log^2 T} \notag \\
    & \leq   \PP(\failflag=1\mid E,1'=1)\Pr(E\mid 1'=1)+\Pr(E^c \mid 1'=1) +\frac{1}{\log^2 T} \notag \\
    & \leq  \PP(\failflag=1\mid E_1)+\frac{2}{T\Delta_i^2}+ \frac{1}{\log^2 T},
\end{align}
where the first and third inequalities are due to the law of total probability, the second inequality is due to \eqref{eq:karm-e5}, and the last inequality is due to \eqref{eq:top-k-mu'-bound}. Let $\Delta'_i=\mu'-\EE[Y^i_1]$, $W_r={\mu'}-Y^i_{r+\tau_1}-\Delta'_i$.  We have that conditioned on $E_1$, $\sum_{r}^sW_r/s=\mu'-\theta_{2',s}-\Delta'$ and $W_r$ is 1-subgaussian with zero mean. By the third statement of Lemma~\ref{lem:colt17}, we have
\begin{align}\label{eq:2arm_unknowngap_term3-top-k-1}
 \PP(\failflag=1 \mid E_1)   
&\leq  \PP\bigg(\exists t_i\geq 1, \mu'-\theta_{i',t_i}+\sqrt{\frac{2}{t_i}\log\Big(\frac{T}{t_i}\Big(\log^2\frac{T}{t_i}+1\Big)\Big) }  \leq 0 \ \bigg | E_1  \bigg) \notag\\
 &\leq   \frac{4(16e^2+1)}{T(1-\epsilon_i)^2\Delta_i^2}. 
\end{align}
For term $T\mathbb{P}(\failflag=1,a=i)$, we choose large enough $T$ to ensure 
\begin{align}
\label{kassume3}
    \exp(-\Delta_i^2\log^2 T/4)\leq \frac{1}{T}.
\end{align}
Then, following the similar  argument in ~\eqref{eq:lagretau2},  we can obtain
\begin{align}\label{eq:fail_eq_1_a_eq_i}
    \PP(\failflag=1,a=i)\leq \frac{1}{T}.
\end{align}
Therefore, substituting \eqref{eq:fail_eq_1}, \eqref{eq:2arm_unknowngap_term3-top-k-1} and \eqref{eq:fail_eq_1_a_eq_i} into the definition of $I_4$ in \eqref{eq:unknowndeltau_regret_decomp-top-k}, we have 
\begin{align}
\label{eq:karmt-I4}
    I_4=\log^2 T\PP(\failflag=1)+T \PP(\failflag=1,a=i) \notag \\
    \leq 2+\frac{2\log^2 T}{T\Delta_i^2}+ \frac{4(16e^2+1)\log^2 T}{T(1-\epsilon_i)^2\Delta_i^2}.
\end{align}

\noindent\textbf{Completing the proof:} we can choose a  sufficiently large $T$ such that all the conditions ~\eqref{kassume1}, \eqref{kassume2}, \eqref{kassume3} are satisfied simultaneously. Substituting~\eqref{eq:karmt-I4}, \eqref{eq:unknone-bound-I_3-top-k}, \eqref{eq:expectedtau_2-top-k} and $I_1\leq 1$ back into \eqref{eq:unknowndeltau_regret_decomp-top-k}, we have
\begin{align*}
    \EE[T_i]\leq 4+\frac{C+2\log(4T\Delta_i^2(\log^2(4T\Delta_i^2)+1)) +\sqrt{4\pi\log(4T\Delta_i^2(\log^2(4T\Delta_i^2)+1))}}{(1-\epsilon_i)^2\Delta_i^2},
\end{align*}
for all $i\geq 2$, where $C>0$ is a universal constant. Note that for $T\rightarrow \infty$,  $\epsilon_i\leq 1/(\log T)^{\frac{1}{3}}$. Hence we have $\lim_{T\rightarrow \infty} \EE[T_i]/\log T=2/\Delta_i^2$ and $\lim_{T\rightarrow \infty} R_{\mu}(T)/\log T=\sum_{i}2/\Delta_i$.
\end{proof}

\section{Proof of the Regret Bound of Algorithm \ref{alg:double-exploration-unknown-gap_anytime}}
\label{sec:detc-anytime}
Now we provide the proof of the regret bound of the anytime version DETC algorithm.
\begin{proof}
The regret of Algorithm \ref{alg:double-exploration-unknown-gap_anytime} is caused by pulling the suboptimal arm $2$, which gives rise to $ R_{\mu}(T)=\sum_{t=1}^{T}\Delta\EE[\ind\{A_t=2\}]$. We will consider two intermediate points $t=\sqrt{\log T}$ and $t=\log^2 T$. Then we can decompose the regret of Algorithm \ref{alg:double-exploration-unknown-gap_anytime} as follows:
\begin{align}\label{eq:anytime-regret-decomp}
    R_{\mu}(T)=\underbrace{\sum_{t=1}^{\sqrt{\log T}}\Delta\EE[\ind\{A_t=2\}]}_{I_1}+\underbrace{\sum_{t=\sqrt{\log T}}^{\log^2 T}\Delta\EE[\ind\{A_t=2\}]}_{I_2}+\underbrace{\sum_{t=\log^2 T}^{T}\Delta\EE[\ind\{A_t=2\}]}_{I_3}.
\end{align}
In what follows, we will bound these terms separately.

\noindent\textbf{Bounding term $I_1$:}
Since the horizon length in this part is only $\sqrt{\log T}$, we can directly upper bound it as $I_1\leq \Delta\sqrt{\log T}$.

\noindent\textbf{Bounding term $I_2$:} Since Algorithm \ref{alg:double-exploration-unknown-gap_anytime} has multiple epochs, we will rewrite the regret in $I_2$ in an epoch-wise fashion. Specifically, without loss of generality, we assume that there are two integers $r_1$ and $r_2$ such that $2^{r_1}=\sqrt{\log T}$ and $2^{r_2}=\log^2 T$ respectively. Denote  $\tau_{2,r}$ to be the total number of pulls of arm 2 in the $r$-th epoch, $r=1,2,\ldots$. Then we can rewrite $I_2$ in the following way:
\begin{align}\label{eq:anytime-I2-decomp}
    I_2=\sum_{t=\sqrt{\log T}}^{\log^2 T}\Delta\EE[\ind\{A_t=2\}]=\sum_{r=r_1}^{r_2}\Delta\EE[\tau_{2,r}].
\end{align}
Note that the chosen arm $1'$ may be different in different epochs. In order to make the presentation more precise, we use $1'(r)$ to denote the arm that is chosen by Line~\ref{algline-anytime-choose-best-arm} in the $r$-th epoch of Algorithm \ref{alg:double-exploration-unknown-gap_anytime}. Also note that at the beginning of epoch $r$, the current time step of the algorithm is $t=2^{r}$. Let $\epsilon_T=1/\log\log T$ and define event
\begin{align*}
     E=\Bigg\{\bigcap_{r=r_1,r_1+1,\ldots,r_2} \big\{ |\hat{\mu}_{1'(r)}(2^r)-\mu_{1'(r)}|< \epsilon_T\Delta \big\}  \Bigg\}.
\end{align*}
Event $E$ essentially says that at the beginning of any epoch $r\in[r_1,r_2]$, the average reward of the chosen arm $1'(r)$ is always close to its mean reward within a margin $\epsilon_T\Delta$. The characterization of this event is the key to analyzing the number of suboptimal arms pulled in each epoch. 

Now we compute the probability that event $E$ happens. For any $r\in[r_1,r_2]$,  at the beginning of the $r$-th epoch, we know that the algorithm has run for $2^r$ times steps. Recall the definition of $T_k(t)$, the number of times that arm $1'(r)$ is pulled is $T_{1'(r)}(2^r)$.  Since arm $1'(r)$ is  the arm that has been pulled for the most times so far, it must have been pulled for more than $2^{r-1}\geq2^{r_1-1}=\sqrt{\log T}/2$ times, namely, $T_{1'(r)}(2^r)\geq\sqrt{\log T}/2$. By Lemma \ref{lem:subguassian}, we have
\begin{align}
\label{eq:anytime-bound-1'}
    \PP\big(\big|\hat{\mu}_{1'}\big(2^r\big)-\mu_{1'}\big|\geq \epsilon_T\Delta\big)&\leq 2\exp\bigg(-\frac{T_{1'(r)}(2^r)\epsilon_T^2\Delta^2}{2}\bigg)\notag\\
    &\leq 2\exp\bigg(-\frac{\sqrt{\log T}\epsilon_T^2\Delta^2}{4}\bigg)\notag\\
    &\leq \frac{2}{\log^4 T},
\end{align}
where the last inequality holds due to $\epsilon_T=1/\log\log T$ and when 
$T$ is sufficiently large $T$ such that
\begin{align}
\label{eq:con-anytime-1}
    \frac{\sqrt{\log T}}{4\log\log T}\geq \frac{4(\log\log T)^2}{\Delta^2}.
\end{align}
Let $E^c$ be the complement of event $E$. Then it holds that
\begin{align}
\label{eq:anytime-E^c-1}
    \PP(E^c) &=\PP\bigg(\bigg\{\bigcap_{r=r_1,r_1+1,\ldots,r_2} \big\{ |\hat{\mu}_{1'}(2^r)-\mu_{1'}|<\epsilon_T\Delta \big\}  \bigg\}^c\bigg) \notag \\
    &=\PP\bigg(\bigcup_{r=r_1,r_1+1,\ldots,r_2} \big\{ |\hat{\mu}_{1'}(2^r)-\mu_{1'}|\geq\epsilon_T\Delta \big\}  \bigg) \notag \\    
    & \leq \sum_{r=r_1}^{r_2}\PP(|\hat{\mu}_{1'}(2^r)-\mu_{1'}|\geq \epsilon_T\Delta)\notag\\
    &\leq 1/\log^3 T,
\end{align}
where in the first inequality we applied the union bound over all epochs $r\in[r_1,r_2]$, and the last inequality is due to \eqref{eq:anytime-bound-1'} and  $r_2=2\log_2\log T\leq \log T/2$ for sufficiently large $T$. 

Based on the characterization of event $E$, we bound the summation of $\tau_{2,r}$ in \eqref{eq:anytime-I2-decomp} as follows:
\begin{align}
\label{eq:boundtau2r}
    \sum_{r=r_1}^{r_2}\EE[\tau_{2,r}]&\leq \sum_{r=r_1}^{r_2}\EE[\tau_{2,r}|E]\PP(E)+\sum_{r=1}^{r_2}2^r\PP(E^c) \notag \\
    &\leq \sum_{r=r_1}^{r_2}\EE[\tau_{2,r}|E]\PP(E)+\frac{2^{r_2+1}}{\log^3 T} \notag \\
    & \leq \sum_{r=r_1}^{r_2}\EE[\tau_{2,r}|E]\PP(E)+\frac{2}{\log T}. 
\end{align}
where in the first inequality we used the fact the $\tau_{2,r}$ is at most $2^r$ in the $r$-th epoch, the second inequality is due to  \eqref{eq:anytime-E^c-1}, and the last inequality is due to $2^{r_2}=\log^2 T$. 

In the $r$-th epoch of Algorithm \ref{alg:double-exploration-unknown-gap_anytime}, $\tau_{2,r}$ is contributed by two part: the number of pulls of arm 2 in Line \ref{line-alg5-5} and the number of pulls of arm 2 in Line \ref{line-alg5-8}. We denote them as $c_r^+$ and $c_r^-$ respectively such that $\tau_{2,r}=c_r^++c_r^-$. In epoch $r\in[r_1,r_2]$, by the fact that $\EE[x]=\sum_{s}\PP(x>s)$ we have
\begin{align}
\label{eq:bounding-cr+-I_2}
   &\EE[c^+_r|E]\PP(E)\notag\\
   &= \sum_{s=1}^{T}\PP(c^+_r \geq s \mid E)\PP(E) \notag \\
   &\leq\sum_{t=2^r}^{2^{r+1}} \PP\Bigg(\hat\mu_1(t)-\hat{\mu}_{2}(t)\leq \sqrt{\frac{2}{T_{2}(t)}\log\bigg(\frac{r\cdot 2^r}{T_{2}(t)}\bigg(\log^2\bigg(\frac{r\cdot 2^r}{T_{2}(t)}\bigg)+1\bigg)\bigg)}  \ \Bigg|E\Bigg)\PP(E) \notag\\   
   &\leq\sum_{t=2^r}^{2^{r+1}} \PP\Bigg(\mu_1-\epsilon_T\Delta-\hat{\mu}_{2}(t)\leq \sqrt{\frac{2}{T_{2}(t)}\log\bigg(\frac{r\cdot 2^r}{T_{2}(t)}\bigg(\log^2\bigg(\frac{r\cdot 2^r}{T_{2}(t)}\bigg)+1\bigg)\bigg)}  \ \Bigg|E\Bigg)\PP(E)  \notag\\
 &\leq \sum_{t=1}^{2^{r+1}} \PP\bigg(\hat{\mu}_{2}(t)-\mu_1+\Delta+ \sqrt{\frac{2}{T_{2}(t)}\log\bigg(\frac{r\cdot 2^r}{T_{2}(t)}\bigg(\log^2(\frac{r\cdot 2^r}{T_{2}(t)})+1\bigg)\bigg)}\geq(1-\epsilon_T)\Delta  \bigg),
\end{align}
where in the first inequality, $c_r^+> 0$ (arm 2 is pulled in Line \ref{line-alg5-5}) means the arm chosen in this epoch is arm $1'(r)=2$ and the stopping condition in Line \ref{line-stopping} of Algorithm~\ref{alg:double-exploration-unknown-gap_anytime} is satisfied, in the second inequality, we used the fact that conditioned on event $E$, it holds that $\hat\mu_1(t)\geq\mu_1-\epsilon_T\Delta$, and in the last inequality, we used the fact that $\PP(x|y)\PP(y)=\PP(x,y)\leq\PP(x)$ for any random variables $x$ and $y$. Now note that $\hat\mu_2(t)-\mu_1+\Delta=\hat\mu_2(t)-\mu_2$ is $1$-subgaussian with zero mean. Applying the second statement of Lemma~\ref{lem:colt17} with $\delta=(1-\epsilon_T)\Delta$, we have
\begin{align}\label{eq:beforetau2-cr+}
    \sum_{r=r_1}^{r_2}\EE[c^+_r|E]\PP(E)&=   \sum_{r=r_1}^{r_2} O\bigg(\frac{\log(r\cdot 2^r\Delta^2)}{(1-\epsilon_T)^2\Delta^2} \bigg) \notag \\
& =2\log \log T\cdot  O\bigg(\frac{\log(r_2\cdot 2^{r_2}\Delta^2)}{\Delta^2} \bigg) \notag \\
& =O(\sqrt{\log T}),
\end{align}
where the last equality is from the upper bound of $1/\Delta^2$ in \eqref{eq:con-anytime-1}  and the following upper bound of $\Delta^2$:
 \begin{align}
 \label{eq:con-anytime-2}
     \Delta\leq \log T \qquad \text{ and }  \qquad \log \log T\geq 4,
 \end{align} 
which holds for sufficiently large $T$.

Now we bound $\EE[c_r^-|E]\PP(E)$. Note that when $c_r^->0$ (arm 2 is pulled in Line \ref{line-alg5-8}), we know that (1) the stopping condition in Line \ref{line-stopping} of Algorithm \ref{alg:double-exploration-unknown-gap_anytime} is violated by some $t\leq2^{r+1}$; and (2) arm $a(r)=2$.  Therefore, we have
\begin{align}\label{eq:anytime-c_r^-}
  \EE[c^-_r|E]\PP(E)
  &=\EE[c^-_r|E,c_r^->0]\PP(E)\PP(c_r^->0|E)\notag\\
  &=\EE[c^-_r|E,c_r^->0]\PP(E)\big[\PP(c_r^->0, 1'=1\mid E)+\PP(c_r^->0, 1'=2\mid E)\big].
\end{align}
For the first term in \eqref{eq:anytime-c_r^-}, similar to the proof in \eqref{eq:bounding-cr+-I_2}, we have
\begin{align}
\label{eq:boundingterm-cr-2I_2}
   &\EE[c^-_r|E,c_r^->0]\PP(E)\PP(c_r^->0, 1'=1\mid E) \notag\\
   &\leq 2^r  \PP\bigg(\exists t\leq 2^{r+1}: {\mu}_1-\epsilon_T\Delta-\hat{\mu}_{2}(t)  <-\sqrt{\frac{2}{T_{2}(t)}\log\bigg(\frac{r\cdot 2^r}{T_{2}(t)}\bigg(\log^2\bigg(\frac{r\cdot 2^r}{T_{2}(t)}\bigg)+1\bigg)\bigg)} \bigg)\notag\\
   &\leq 2^{r+2}O\bigg(\frac{1}{r\cdot 2^r\Delta^2}\bigg),
\end{align}
where in the first inequality we used the fact that $c_r^-\leq 2^r$ and $\hat\mu_1(t)\geq\mu_1-\epsilon_T\Delta$, and the second inequality is due to third statement of Lemma~\ref{lem:colt17}. Using exactly the same argument, we have
\begin{align}
\label{eq:boundingterm-cr--I_2}
   &\EE[c^-_r|E,c_r^->0]\PP(E)\PP(c_r^->0, 1'=2\mid E) \notag\\
   &\leq 2^r  \PP\bigg(\exists t\leq 2^{r+1}:  {\mu}_2+\epsilon_T\Delta-\hat{\mu}_{1}(t)  >\sqrt{\frac{2}{T_{2}(t)}\log\bigg(\frac{r\cdot 2^r}{T_{2}(t)}\bigg(\log^2\bigg(\frac{r\cdot 2^r}{T_{2}(t)}\bigg)+1\bigg)\bigg)} \bigg)\notag\\
   &\leq 2^{r+2}O\bigg(\frac{1}{r\cdot 2^r\Delta^2}\bigg).
\end{align}
Therefore, it holds that
\begin{align}
    \label{eq:bounding-term-c_r^-}
    \EE[c_r^-|E]\PP(E)=O\bigg(\frac{1}{r\Delta^2} \bigg).
\end{align}
We further have
\begin{align}
\label{eq:anytime-c_r^-final}
   \sum_{r=r_1}^{r_2}\EE[c_r^-|E]\PP(E)
    =  &  \sum_{r=r_1}^{r_2} O\bigg(\frac{1}{r\cdot \Delta^2}\bigg) 
   =  O\bigg(\frac{\log r_2}{r_1\Delta^2}\bigg)=O\bigg(\frac{1}{\Delta^2}\bigg).
\end{align}
Combining \eqref{eq:boundtau2r},  \eqref{eq:beforetau2-cr+} and \eqref{eq:anytime-c_r^-final} together, we have
\begin{align}
\label{eq:T_2(tau)}
    I_2&=\Delta\sum_{r=r_1}^{r_2}\EE[\tau_{2,r}]\notag\\
    &=\Delta\sum_{r=r_1}^{r_2}\EE[c_r^+|E]\PP(E)+\Delta\sum_{r=r_1}^{r_2}\EE[c_r^-|E]\PP(E)+\frac{2\Delta}{\log T}\notag\\
    &=O\bigg(\Delta\sqrt{\log T}+\frac{1}{\Delta}+\frac{\Delta}{\log T}\bigg).
\end{align}

\noindent{\bf Bounding term $I_3$:} 
We start with decomposing $I_3$ into two terms. The first term is the number of pulls of arm 2 at Line~\ref{line-alg5-5}, i.e., $\sum_{r=r_2+1}^{\log_2 T}c_r^+$ and the second term is the number of pulls of arm 2 at Line~\ref{line-alg5-8}, i.e., $\sum_{r=r_2+1}^{\log_2 T}c_r^-$. Therefore, we have
\begin{align}
\label{eq:anytime-I_3}
  I_3&  =\EE\Bigg[\sum_{r=r_2+1}^{\log_2 T}c_{r}^-\Bigg]+\EE\Bigg[\sum_{r=r_2+1}^{\log_2 T}c_{r}^+\Bigg].
\end{align}
Define event 
\begin{align*}
     E'=\Bigg\{\bigcap_{r=r_2+1,\ldots,\log_2 T} \big\{ |\hat{\mu}_{1'(r)}(2^r)-\mu_{1'(r)}|< \epsilon_T\Delta \big\}  \Bigg\}.
\end{align*}
$E'$ says that for epoch $r\geq r_2+1$, the average reward of  $1'(r)$ is close to its mean reward within a margin $\epsilon_T\Delta$, which  plays a similar role as $E$ does. 
Now, we compute the probability that $E'$ happens. Since arm $1'(r)$ is the arm that has been pulled for the most times so far, $1'(r)$ have been pulled for more than $2^{r_2}\geq \log^2 T$ times. By Lemma \ref{lem:subguassian}, we have
\begin{align}
\label{eq:anytime-bound-1'-T}
    \PP\big(\big|\hat{\mu}_{1'}\big(2^r\big)-\mu_{1'}\big|\geq \epsilon_T\Delta\big)&\leq 2\exp\bigg(-\frac{T_{1'(r)}(2^r)\epsilon_T^2\Delta^2}{2}\bigg)\notag\\
    &\leq 2\exp\bigg(-\frac{\log^2 T\epsilon_T^2\Delta^2}{2}\bigg)\notag\\
    &\leq \frac{2}{T^2},
\end{align}
where the last inequality holds due to $\epsilon_T=1/\log\log T$ and when 
$T$ is sufficiently large $T$ such that
\begin{align}
\label{eq:con-anytime-1-T}
    \frac{\log^2 T}{2\log T}\geq \frac{2(\log\log T)^2}{\Delta^2}.
\end{align}
Let $E'^c$ be the complement of event $E'$. Then it holds that
\begin{align}
\label{eq:anytime-E^c-1-T}
    \PP(E'^c) &=\PP\bigg(\bigg\{\bigcap_{r=r_2+1,\ldots,\log_2 T} \big\{ |\hat{\mu}_{1'}(2^r)-\mu_{1'}|<\epsilon_T\Delta \big\}  \bigg\}^c\bigg) \notag \\
    & \leq \sum_{r=r_2+1}^{\log_2 T}\PP(|\hat{\mu}_{1'}(2^r)-\mu_{1'}|\geq \epsilon_T\Delta)\notag\\
    &\leq 1/T,
\end{align}
where in the first inequality we applied the union bound over all epochs $r\in [r_2+1,\log_2 T]$, and the last inequality is due to \eqref{eq:anytime-bound-1'-T} and  $\log_2 T\leq T/2$. 
 Based on the characterization of event $E'$, we bound the summation of $c_r^+$ and $c_r^-$ in \eqref{eq:anytime-I_3} as follows:
\begin{align}
\label{eq:boundtau2r-T}
    \sum_{r=r_2+1}^{\log_2 T}\EE[c_r^+]+\sum_{r=r_2+1}^{\log_2 T}\EE[c_r^-] 
&\leq  \sum_{r=r_2+1}^{\log_2 T}\EE[c_r^+\mid E']\PP(E')+\sum_{r=r_2+1}^{\log_2 T}\EE[c_r^-\mid E']\PP(E')+\sum_{r=r_2+1}^{\log_2 T}2^r\PP(E'^c) \notag \\
&\leq  \sum_{r=r_2+1}^{\log_2 T}\EE[c_r^+\mid E']\PP(E')+\sum_{r=r_2+1}^{\log_2 T}\EE[c_r^-\mid E']\PP(E')+2, 
\end{align}
where the second inequality is due to  \eqref{eq:anytime-E^c-1-T}. Now, we  bound term $\sum_{r=r_2+1}^{\log_2 T}\EE[c^+_r|E']\PP(E')$.  Using the previous results~\eqref{eq:bounding-cr+-I_2} of bounding $\EE[c^+_r|E]\PP(E)$, we have
\begin{align}
    & \sum_{r=r_2+1}^{\log_2 T}\EE[c^+_r|E']\PP(E') \notag \\
&\leq  \sum_{r=r_2+1}^{\log_2 T}\sum_{t=2^r}^{2^{r+1}} \PP\Bigg(\hat{\mu}_{2}(t)-\mu_1+\Delta+ \sqrt{\frac{2}{T_{2}(t)}\log\bigg(\frac{r\cdot 2^r}{T_{2}(t)}\bigg(\log^2\bigg(\frac{r\cdot 2^r}{T_{2}(t)}\bigg)+1\bigg)\bigg)}\geq(1-\epsilon_T)\Delta  \Bigg) \notag \\
&\leq  \sum_{t=1}^{T} \PP\Bigg(\hat{\mu}_{2}(t)-\mu_1+\Delta+ \sqrt{\frac{2}{T_{2}(t)}\log\bigg(\frac{\log_2 T\cdot T}{T_{2}(t)}\bigg(\log^2\bigg(\frac{\log_2 T\cdot T}{T_{2}(t)}\bigg)+1\bigg)\bigg)}\geq(1-\epsilon_T)\Delta  \Bigg) \notag \\
 & \leq \frac{2\log(T\Delta^2\log T)+o(\log (T\Delta^2\log T))}{(1-\epsilon_T)^2\Delta^2},
 \end{align}
where the last inequality is due to  the second statement of Lemma~\ref{lem:colt17}. Now, we turn to bounding term $\sum_{r=r_2
+1}^{\log_2 T}\EE[c^-_r|E']\PP(E')$. Using exactly the same argument on bounding term $\EE[c^-_r|E]\PP(E)$ in \eqref{eq:bounding-term-c_r^-}, we have $\EE[c^-_r|E']\PP(E')=O(1/(r\Delta^2))$. 
Therefore, we have
\begin{align}
\label{eq:anytime-c_r^--}
   \sum_{r=r_2+1}^{\log_2 T}\EE[c_r^-|E']\PP(E')
    =    \sum_{r=r_2+1}^{\log_2 T} O\bigg(\frac{1}{r\cdot \Delta^2}\bigg) 
   =   O\bigg(\int_{x=1}^{\log T} \frac{1}{x\Delta^2}\bigg) \dd x = O\bigg(\frac{\log\log T}{\Delta^2}\bigg).
\end{align}
Combing \eqref{eq:boundtau2r-T}  and \eqref{eq:anytime-c_r^--} together, we have 
\begin{align}
\label{eq:anytime-T-tau2}
    I_3= \frac{2\log(T\Delta^2\log T)+o(\log (T\Delta^2\log T))}{(1-\epsilon_T)^2\Delta^2}+O\bigg(\frac{\log\log T}{\Delta^2}\bigg)+O(1).
\end{align}
Substituting \eqref{eq:T_2(tau)} and \eqref{eq:anytime-T-tau2} into \eqref{eq:anytime-regret-decomp}, we  have 
\begin{align*}
    \lim_{T\rightarrow \infty}\frac{R_{\mu}(T)}{\log T}=\frac{2}{\Delta},
\end{align*}
which completes the proof.
\end{proof}

\section{Round Complexity of Batched DETC}
In this section, we derive the round complexities of Algorithms \ref{alg:batched_DETC_known_gap} and \ref{alg:batched_DETC_unknown_gap} for batched bandit models. We will prove that Batched DETC still enjoys the asymptotic optimality. Note that in batched bandits, our focus is on the asymptotic regret bound and thus we assume that $T$ is sufficiently large throughout the  proofs in this section to simplify the presentation. 

\subsection{Proof of Theorem \ref{thm:knowngap_batch}}\label{sec:proof_batch_detc_known_gap_round}
We first prove the round complexity for Batched DETC (Algorithm \ref{alg:batched_DETC_known_gap}) when the gap $\Delta$ is known.
\begin{proof}
The analysis is very similar to that of Theorem~\ref{theorem:knowndelta_2arm} and thus we will use the same notations therein. Note that $\stageOne$ requires $1$ round of queries since $\tau_1$ is fixed. In addition, $\stageTwo$ and $\stageFour$ need $1$ query at the beginning of stages respectively. Now it remains to calculate the total rounds for $\stageThree$. 

 Recall that $E$ is event $\mu'\in [\mu_{1'}-\epsilon_T\Delta,\mu_{1'}+\epsilon_T\Delta]$, $E_1=\{E,1'=1\}$ and $E_2=\{E,1'=2\}$. We first assume that $E_1$ holds. Let $x_i={i(2\sqrt{\log(T\Delta^2)}+4)}$ and $n_{x_i}=\tau_0+x_i/(2(1-\epsilon_T)^2\Delta^2)$. For simplicity, assume $x_i,n_{x_i} \in \NN^+$.  From ~\eqref{eq:roundused1}, we have 
\begin{equation}\label{eq:test1}
 \begin{split}
  \PP(\tau_2> n_{x_i} \mid E_1)\leq  \PP\bigg(S_{n_{x_i}}\leq \frac{\log(T\Delta^2)}{2(1-\epsilon_T)\Delta}  \;\bigg|\; E_1\bigg) & \leq  \exp\bigg(-\frac{x_i^2}{4(\log(T\Delta^2)+x_i)}
     \bigg) \\
     & \leq \exp \bigg(-\frac{x_i}{2\sqrt{\log (T\Delta^2)}+4} \bigg) \\
     & \leq 2^{-i}.
     \end{split}
\end{equation}
Thus, the expected number of rounds of queries needed in $\stageThree$ of Algorithm \ref{alg:batched_DETC_known_gap} is upper bounded by $\sum_{i=1}^{\infty} i/2^i=2$. Similarly, if $E_2$ holds, we still have the expected number of rounds in {\it Stage III} is upper bounded by 2. Lastly, if $E^c$ holds, we have $\PP(E^c)\leq 2/(T\Delta^2)$.  Note that the increment between consecutive test time points is $(2\sqrt{\log(T\Delta^2)}+4)/(2(1-\epsilon_T)^2\Delta^2)$, thus the expected number of  test time points is at most $T(1-\epsilon_T)^2\Delta^2/({\sqrt{\log(T\Delta^2)}})$. Then the expected number of rounds for this case is bounded by $2(1-\epsilon_T)^2/(\sqrt{\log(T\Delta^2)})$. For $T\rightarrow \infty$, the expected number of rounds cost for this case is $0$. To summarize, the round complexity of Algorithm \ref{alg:batched_DETC_known_gap} is $O(1)$.

Following the same proof in \eqref{eq:geqyregret} and \eqref{eq:regrettau_2}, it is easy to verify that $\mathbb{E}[\tau_2\mid E_1]\leq \tau_0+(2\sqrt{\log(T\Delta^2)}+4)/((1-\epsilon_T)^2\Delta^2)$, which is no larger than the bound in \eqref{eq:regrettau_2}. 
The bounds for other terms remain the same.
Therefore, the batched version of Algorithm~\ref{alg:batched_DETC_known_gap} is still asymptotically optimal, instance-dependent optimal and minimax optimal. 
\end{proof}

\subsection{Proof of Theorem \ref{thm:unknowngap_batch}}\label{sec:proof_batch_detc_unknown_gap_round}
Now we  prove the round complexity and regret bound for Batched DETC (Algorithm \ref{alg:batched_DETC_unknown_gap}) when the gap $\Delta$ is unknown.
\begin{proof}
For the sake of simplicity, we use the same notations that are used in Theorem~\ref{theorem:unknowndelta} and its proof. To compute the round complexity and regret of $\stageOne$, we first compute the probability that $\tau_1> 2i \sqrt{\log T}$.
 We assume $T$ is large enough such that it satisfies 
\begin{equation}
\label{eq:assume1}
\sqrt{\log T}\geq 16\log^{+}(T_1\Delta^2/2)/\Delta^2,
\end{equation}
where we recall that $T_1=\log^2 T$. Let $s_i=2i \sqrt{\log T}$ for $i=1,2,\ldots$ and  $\gamma=4\log^{+}(T_1\Delta^2/2)/\Delta^2$. From ~\eqref{eq:assume1}, it is easy to verify that $s_i\geq 32i/\Delta^2$, $\gamma/s_i \leq 1/8 $ and
$\sqrt{4\log^{+} (T_1/2s_i)/s_i}\leq \Delta \sqrt{\gamma/s_i}$. The stopping rule in $\stageOne$ implies
\begin{align*}
    \PP ( \tau_1 \geq s_i) & \leq \PP \Bigg( \hat{\mu}_{1,s_i}-\hat{\mu}_{2,s_i} \leq \sqrt{\frac{8}{s_i}\log^{+}\bigg(\frac{T_1}{2s_i}\bigg)} \Bigg) \\
    &  = \PP\bigg( \frac{\sum_{i=1}^{s_i}Z_i}{s_i}\leq \sqrt{\frac{4}{s_i}\log^{+}\bigg(\frac{T_1}{2s_i}\bigg)}-\frac{\Delta}{\sqrt{2}}\bigg) \\
    &  \leq \PP \bigg( \frac{\sum_{i=1}^{s_i}Z_i}{s_i}\leq \Delta\sqrt{\frac{\gamma}{s_i}}-\frac{\Delta}{\sqrt{2}} \bigg) \\
    &   \leq \exp \bigg( -\frac{s_i\Delta^2}{2}\bigg( \frac{1}{\sqrt{2}}-\sqrt{\frac{\gamma}{s_i}}\bigg)^2 \bigg) \\
    &   \leq \exp(-i) \\
    &\leq 2^{-i},
\end{align*}
where 
the third inequality follows from Lemma~\ref{lem:subguassian} and the fourth inequality is due to the fact that  $s_i\geq 32i/\Delta^2$ and $\gamma/s_i \leq 1/8 $. Hence by the choice of testing points in \eqref{eq:testtt}, the expected number of rounds needed in $\stageOne$ of Algorithm \ref{alg:batched_DETC_unknown_gap} is upper bounded by $\sum_{i=1}^{\infty}i/2^i\leq 2$. The expectation of $\tau_1$ is upper bounded by $\mathbb{E}[\tau_1]\leq \sum_{i=1}^{\infty}2i\sqrt{\log T}/2^i\leq 4\sqrt{\log T}$, which matches the bound derived in \eqref{eq:unknowndeltau_1}.

Now we focus on bounding term $\Delta\EE[\tau_2]$ and the round complexity in $\stageThree$. Let $\epsilon'_T=\sqrt{2}\epsilon_T=\sqrt{4\log(T\Delta^2)/{(T_1\Delta^2)}}$. Let $E$ be the event $\mu'\in[\mu_{1'}-\epsilon'_T\Delta,\mu_{1'}+\epsilon'_T\Delta]$.
 Applying Lemma~\ref{lem:subguassian}, we have $\PP(E^c)\leq 1/(T^2\Delta^4)$. Hence, the expected number of test time points contributed by case $E^c$  is $O(1/(T\Delta^4))$ which goes to zero when $T\rightarrow\infty$. Similarly, we assume that $E$ holds and the chosen arm $1'=1$. Recall $E_1=\{E, 1'=1 \}$.  Recall that this condition also implies $\Delta'\in [(1-\epsilon'_T)\Delta, (1+\epsilon'_T)\Delta]$, where $\epsilon'_T=\sqrt{\log(T\Delta^2)/(T_1\Delta^2)}$ and $T_1=\log^2 T$. When $T$ is large enough such that it satisfies
\begin{align} \label{eq:assume2} 
\sqrt{\frac{4\log(T\Delta^2)}{\Delta^2\log^2 T}}\leq \frac{1}{(\log T)^{\frac{1}{3}}},
\end{align} 
we have $\epsilon'_T\leq 1/ (\log T)^{\frac{1}{3}}$. Furthermore, we can also choose a large $T$ such that
\begin{equation}
\label{eq:assume3}
 \sqrt{\log T} (\Delta')^2 \geq 2 (\log \log T)^2.
\end{equation}
Applying Lemma~\ref{lem:subguassian}, we have
\begin{align}\label{eq:theta_2prime_concentrate}
    \PP\Big(\mu_{2'}-{\Delta'}{{(\log T)}^{-\frac{1}{4}}}\leq\theta_{2',N_1}\leq \mu_{2'}+{\Delta'}{{(\log T)}^{-\frac{1}{4}}}\mid E_1\Big)& \geq 1-2\exp \bigg( -\frac{2\log T(\Delta')^2}{2\sqrt{\log T}\log \log T}\bigg)  \notag\\
    & \geq 1-\frac{2}{\log^2 T},
\end{align}
where the last inequality follows by~\eqref{eq:assume3}. This means that after the first round of $\stageThree$ in Algorithm \ref{alg:batched_DETC_unknown_gap}, the average reward for arm $2'$ concentrates around the true value $\mu_{2'}$ with a high probability. Let $E_3$ be the event $\mu_{2'}-{\Delta'}/{\sqrt[4]{\log T}}\leq \theta_{2',N_1}\leq \mu_{2'}+{\Delta'}/{\sqrt[4]{\log T}}$.
Recall that $E_1=\{E,1'=1\}$ and $E_2=\{E,1'=2\}$. Let $H_1=\{E_1,E_3\}$ and $H_2=\{E_2,E_3\}$. We have
\begin{align}
    \EE[\tau_2]& \leq \EE[\tau_2\mid E_1,E_3]\PP[E_1,E_3]+\EE[\tau_2\mid E_2,E_3]\PP[E_2,E_3]+\EE[\tau_2\mid E_3^c]\PP[E_3^c]+\EE[\tau_2 \mid E^c]\PP[E^c] \notag \\
    &\leq \EE[\tau_2\mid H_1]\PP[H_1]+\EE[\tau_2\mid H_2]\PP[H_2]+\EE[\tau_2\mid E_3^c]\PP[E_3^c]+ 2/(T\Delta^3)
\end{align}
We first focus on term $\EE[\tau_2\mid H_1]$. We assume event $H_1$ holds. Define
\begin{align*}
s_i'&= \frac{2(1+{1}/{\sqrt[4]{\log T}})^2\log (T\log^3 T)}{\hat{\Delta}^2}+\frac{i(1+{1}/{\sqrt[4]{\log T}})^2(\log T)^{\frac{2}{3}}}{\hat{\Delta}^2},\\
\gamma'&=\frac{2\log\big(T(\Delta')^2[\log^2({T}{(\Delta')^2})+1]\big)}{(\Delta')^2},
\end{align*}
for $i=1,2,\ldots$. Recall the definition of test time points in \eqref{eq:testt_2}, we know that the $(i+1)$-th test in $\stageThree$ happens at time step $t_2=s_i'$. We choose a large enough $T$ such that
\begin{equation}
\label{eq:assume4}
    \log^3 T\geq (\Delta')^2(\log^2 (T(\Delta')^2)+1).
\end{equation}
Let $\Delta'=\mu'-\mu_{2'}$. Hence conditioned on $H_1$, $\hat{\Delta}=\mu'-\theta_{2',N_1}\in  [(1-1/\sqrt[4]{\log T})\Delta', (1+1/\sqrt[4]{\log T})\Delta']$.  
Then we have that conditioned on $H_1$
\begin{equation}
\label{eq:gamma'}
    \frac{2(1+1/\sqrt[4]{\log T})^2\log (T \log^3 T)}{\hat{\Delta}^2}\geq \frac{2\log (T \log^3 T)}{(\Delta')^2} \geq \gamma',
\end{equation}
where the last inequality is due to \eqref{eq:assume4}. On the other hand, we also have that conditioned on $H_1$
 \begin{equation}
     s_i'\geq\frac{2(1+1/\sqrt[4]{\log T})^2\log (T \log^3 T)}{\hat{\Delta}^2}\geq \frac{2}{(\Delta')^2}.
 \end{equation}
Therefore, by the definition of $\gamma'$, it holds that  conditioned on $H_1$
\begin{align*}
    \Delta' \sqrt{\frac{\gamma'}{s_i'}}&=\sqrt{\frac{2}{s_i'}\log(T(\Delta')^2(\log^2(T(\Delta')^2)+1))}\geq\sqrt{\frac{2}{s_i'}\log\bigg(\frac{T}{s_i'}\bigg(\log^2 \bigg(\frac{T}{s_i'}\bigg)+1\bigg)\bigg)}.
\end{align*}
Recall the definition $W_i={\mu'}-Y_{i+\tau_1}-\Delta'$ used in \eqref{eq:expectedtau_2}. From the stopping rule of $\stageThree$ in Algorithm \ref{alg:double-exploration-unknown-gap}, conditioned on $H_1$, we obtain
\begin{align}
\label{eq:1/2^i}
  \PP(\tau_2\geq s_i'\mid H_1) 
        & \leq \PP\bigg(\mu'-\theta_{2',s_i'}\leq \sqrt{\frac{2}{s_i'}\log\Big(\frac{T}{s_i'}\Big(\log^2\frac{T}{s_i'}+1\Big)\Big) } \;\bigg|\; H_1\bigg) \notag \\
        & =\PP\bigg( \frac{\sum_{i=1}^{s_i'}W_i}{s_i'}+\Delta' \leq  \sqrt{\frac{2}{s_i'}\log\Big(\frac{T}{s_i'}\Big(\log^2\frac{T}{s_i'}+1\Big)\Big) } \;\bigg|\; H_1\bigg) \notag \\
        & \leq \exp \bigg( -\frac{s_i'(\Delta')^2}{2} \bigg( 1-\sqrt{\frac{\gamma'}{s_i'}}\bigg)^2 \bigg) \notag \\
        & = \exp \bigg(-\frac{(\Delta')^2}{2} (\sqrt{s_i'}-\sqrt{\gamma'})^2 \bigg)  \notag \\
        & = \exp  \bigg(-\frac{(\Delta')^2}{2} \bigg(\frac{s_i'-\gamma'}{\sqrt{s_i'}+\sqrt{\gamma'}}\bigg)^2 \bigg)  \notag \\
        & \leq \exp \bigg(-\frac{i^2(\log T)^{4/3}}{8s_i'(\Delta')^2}  \bigg) ,
\end{align}
where the second inequality from Lemma~\ref{lem:subguassian} 
and in the last inequality we used the fact that  $s_i'-\gamma'\geq i(1+{1}/{\sqrt[4]{\log T}})^2(\log T)^{\frac{2}{3}}/({\hat{\Delta}^2}) \geq i(\log T)^{\frac{2}{3}}/(\Delta')^2$ by \eqref{eq:gamma'}. Choose sufficiently large $T$ to ensure
\begin{equation}
\label{eq:assume5}
   (\log T)^{\frac{4}{3}}\geq 8s_i'(\Delta')^2.
\end{equation}
Substituting \eqref{eq:assume5} back into~\eqref{eq:1/2^i} yields $\PP(\tau_2\geq s_i' \mid H_1) \leq {1}/{2^i}$. Similarly, $\PP(\tau_2\geq s_i' \mid H_2) \leq {1}/{2^i}$, Thus conditioned on $H_1$ (or $H_2$), the expected rounds used in \emph{Stage III} of Algorithm \ref{alg:double-exploration-unknown-gap} is upper bounded by $\sum_{i=1}^{\infty}i/2^i\leq 2$. Recall that from ~\eqref{eq:assume2}, $\epsilon'_T\leq 1/(\log T)^{\frac{1}{3}}$.
Conditional on $H_1$, the expectation of $\tau_2$ is upper bounded by 
\begin{align}\label{eq:logT^2/3}
\mathbb{E}[\tau_2\mid H_1]&\leq s_1'+\sum_{i=2}[(s_i'-s_{1}')\PP(\tau_2\geq s_i' \mid H_1)]\notag\\
&\leq \frac{2(1+{1}/{(\log T)^{\frac{1}{4}}})^2\log (T \log^3 T)}{\hat{\Delta}^2}+\frac{2(1+{1}/{\sqrt[4]{\log T}})^2(\log T)^{\frac{2}{3}}}{{\hat{\Delta}}^2} \notag\\
&\leq     \frac{2(1+{1}/{(\log T)^{\frac{1}{4}}})^2\log (T \log^3 T)+2(1+{1}/{(\log T)^{\frac{1}{4}}})^2(\log T)^{\frac{2}{3}}}{(1-{1}/{(\log T)^{\frac{1}{3}}})^2(1-{1}/{(\log T)^{\frac{1}{4}}})^2{\Delta}^2},
\end{align}
where the last inequality is due to $\Delta'\in[(1-\epsilon'_T)\Delta,(1+\epsilon'_T)\Delta]$. Similarly, we can derive same bound as in~\eqref{eq:logT^2/3} for $\EE[\tau_2\mid H_2]$. 

For the case $E_3^c$. Note that $\tau_2 \leq \log^2 T$ and we have $\PP(E_3^c)\leq 2/\log^2 T$ by \eqref{eq:theta_2prime_concentrate}. Therefore $\EE[\tau_2 \mid E_3^c]$ can be upper bounded by $2$, which is dominated by \eqref{eq:logT^2/3}.  Since $\tau_2\leq \log^2 T$, conditioned on $E_3^c$, the  expected rounds  is upper bounded by
    $\PP(E_3^c) \cdot  \log^2T\leq 2$. To summarize, we have proved that conditioned on $H_1$ (or $H_2$, or $E^c$, or $E_3^c$), the expected rounds cost is $O(1)$. Therefore, the expected rounds cost of \emph{Stage III} is $O(1)$.


Note that the above analysis does not change the regret incurred in $\stageThree$. A slight difference of this proof from that of Theorem \ref{theorem:unknowndelta} arises when we terminate $\stageThree$ with $t_2=\log^2 T$. The term $I_3$ can be written as 
\begin{align}
 I_3=\Delta T\PP(\tau_2=\log^2 T, a=2)+\Delta T\PP(\tau_2<\log^2 T, a=2),
\end{align}
We can derive same bound as \eqref{eq:unknone-bound-I_3} for term $\Delta T\PP(\tau_2<\log^2 T, a=2)$. Now, we focus on term $\Delta T\PP(\tau_2=\log^2 T, a=2)$. For this case, we have tested $\log^2 T$ samples for both arm $1$ and $2$.  Let $G_0=0$ and $G_n=(X_1-Y_{1+\tau_1})+\cdots+(X_n-Y_{n+\tau_1})$ for every $n\geq 1$. Then $X_i-Y_{i+\tau_1}-\Delta$ is a $\sqrt{2}$-subgaussian random variable. 
Applying Lemma~\ref{lem:subguassian} with $\epsilon=\Delta$ yields
\begin{align*}
     \mathbb{P}\bigg(\frac{G_{\tau_2}}{\tau_2}\leq 0\bigg) \leq \exp \bigg(-\frac{\tau_2\Delta^2}{4}\bigg).
\end{align*}
Conditioned on $\tau_2=\log^2 T$, we further obtain
$\PP(a=2)=\PP(G_{\tau_2}\leq 0)\leq \exp(-\Delta^2\log^2 T /4)\leq1/T$, where in the last inequality we again choose large enough $T$ to ensure 
\begin{equation}
\label{eq:assume6}
\exp(-\Delta^2\log^2 T /4) \leq \frac{1}{T}.
\end{equation}
Therefore, we have proved that conditional on $\tau_2=\log^2 T$,
\begin{equation}
\label{eq:lagretau2}
    \mathbb{P}(a=2) \leq \frac{1}{T}.
\end{equation}
Hence, $\Delta T\PP(\tau_2=\log^2 T, a=2)\leq 1/\Delta$.

To summarize, we can choose a  sufficiently large $T$ such that all the conditions ~\eqref{eq:assume1}, \eqref{eq:assume2}, \eqref{eq:assume3}, \eqref{eq:assume4}, \eqref{eq:assume5} and \eqref{eq:assume6} are satisfied simultaneously. Then the round complexity of Algorithm~\ref{alg:double-exploration-unknown-gap} is $O(1)$. For the regret bound, since the only difference between Algorithm \ref{alg:batched_DETC_unknown_gap} and Algorithm \ref{alg:double-exploration-unknown-gap} is the stopping rules of $\stageOne$ and $\stageThree$,  
we only need to combine the regret for terms  ~\eqref{eq:logT^2/3} and \eqref{eq:lagretau2} and the fact that $\Delta\EE[\tau_1]\leq 4\Delta\sqrt{\log T}$ to obtain the total regret. 
Therefore, we have $\lim_{T\rightarrow \infty}R(T)/\log T =2/\Delta$.
\end{proof}

\section{Proof of Concentration Lemmas}

In this section, we provide the proof of the concentration lemma and the maximal inequality for subgaussian random variables. 

\subsection{Proof of Lemma \ref{lemma:maximal_ineq}}
Our proof relies on the following maximal inequality for supermartingales.
\begin{lemma}[\cite{Ville1939}] If $(S_n)$ is a non-negative supermartingale, then for any $x>0$, 
\begin{align*}
    \PP\bigg(\sup_{n\in \NN}S_n>x\bigg)\leq \frac{\EE [S_0]}{x}.
\end{align*}
\end{lemma}
\begin{proof}[Proof of Lemma \ref{lemma:maximal_ineq}]
The proof follows from the same idea as the proof of Lemma 4 (Maximal Inequality) in \cite{menard2017minimax}. If $\hat{\mu}_n>0$, then \eqref{eq:maximal_ineq} holds trivially. Otherwise, if event $\{\exists N\leq n\leq M, \hat{\mu}_n+\gamma\leq0 \}$ holds, then the following three inequalities also hold simultaneously: 
\begin{align*}
     \hat{\mu}_n\leq0, \qquad -\gamma\hat{\mu}_n-\frac{\gamma^2}{2}\geq \gamma^2-\frac{\gamma^2}{2}=\frac{\gamma^2}{2}, \ \ \ \ \text{and} \ \ \ \ -\gamma n \hat{\mu}_n- \frac{n\gamma^2}{2} \geq \frac{N\gamma^2}{2},
\end{align*}
where the second inequality is due to $\hat{\mu}_n\leq-\gamma$ and the last is due to $n\geq N$. 
Therefore, we have
\begin{align*}
      \PP(\exists N\leq n  \leq M, \hat{\mu}_n+\gamma\leq 0) &\leq \PP\bigg(\exists N\leq n  \leq M, -\gamma n \hat{\mu}_n-\frac{n\gamma^2}{2} \geq \frac{N\gamma^2}{2}\bigg) \notag \\
      & = \PP \bigg( \max_{N\leq n \leq M} \exp \bigg(-\gamma n \hat{\mu}_n-\frac{n\gamma^2}{2}\bigg)\geq \exp \bigg(\frac{N\gamma^2}{2} \bigg)\bigg) \notag \\
        & \leq \PP \bigg( \max_{1\leq n \leq M} \exp \bigg(-\gamma n \hat{\mu}_n-\frac{n\gamma^2}{2}\bigg)\geq \exp \bigg(\frac{N\gamma^2}{2} \bigg)\bigg) \notag \\
     & \leq \frac{\EE[\exp(-\gamma X_1-\gamma^2/2)] } {\exp ( N\gamma^2/2)} \notag \\
      & \leq \exp \bigg( -\frac{N\gamma^2}{2}\bigg),
      \end{align*}
where the third inequality is from Ville's maximal inequality \citep{Ville1939} for non-negative supermartingale and the fact that $S_n=\exp(-\gamma n \hat{\mu}_n-n\gamma^2/2)$ is a non-negative supermartingale. To show $S_n$ is a non-negative supermartingale, we have 
\begin{align*}
       \EE[\exp(-\gamma n\hat\mu_n-n\gamma^2/2)|S_1,\ldots,S_{n-1}] & = S_{n-1}\EE[\exp(-\gamma X_n)]\exp(-\gamma^2/2) \notag \\
      &\leq S_{n-1}\exp(\gamma^2/2)\exp(-\gamma^2/2)  \\
    & \leq S_{n-1},
\end{align*}
where the first inequality is from the definition of 1-subgaussian random variables. This completes the proof.
\end{proof}

\subsection{Proof of Lemma \ref{lem:peeling-simp}}
\begin{proof}
Let $Z_i=(X_i-Y_i-\Delta)/\sqrt{2}$. Then $Z_s$ is a $1$-subgaussian random variable with zero mean. Applying the standard peeling technique, we have
\begin{align}
  &\PP\bigg( \exists s\geq 1: \hat{\mu}_{s} +\sqrt{\frac{8\log^{+}(N/s)}{s}}\leq 0 \bigg) \notag\\
    & \leq \PP \Bigg( \exists s \geq 1: \frac{\sum_{i=1}^{s}Z_i}{s} +\sqrt{\frac{4\log^{+}(N/s)}{s}}+\frac{\Delta}{\sqrt{2}}\leq 0\bigg) \notag\\
& \leq \frac{15}{N\Delta^2},
\end{align}
where the last inequality is from  Lemma 9.3 of \cite{lattimore2018bandit}.
\end{proof}

\subsection{Proof of Lemma \ref{lem:colt17}}
To prove Lemma \ref{lem:colt17}, we also need the following technical lemma from \cite{menard2017minimax}. 
\begin{lemma}
\label{lem:ebeta}
For all $\beta>1$ we have
\begin{equation}
    \frac{1}{e^{\log(\beta)/\beta}-1}\leq 2\max \{\beta, \beta/(\beta-1) \}.
\end{equation}
\end{lemma}

\begin{proof}[Proof of Lemma \ref{lem:colt17}]
For the first statement, let $\gamma=4\log^{+}(T_1\delta^2)/\delta^2$. Note that for $n\geq 1/\delta^2$, it holds that
\begin{equation}
   \delta \sqrt{\frac{\gamma}{n}}=\sqrt{\frac{4}{n}\log^{+}(T_1\delta^2)}\geq\sqrt{\frac{4}{n}\log^{+}\Big(\frac{T_1}{n}\Big)}.
\end{equation}
Let $\gamma'=\max\{\gamma,1/\delta^2\}$.
Therefore, we have
\begin{align}
    \sum_{n=1}^T \PP\bigg( \hat{\mu}_n + \sqrt{\frac{4}{n}\log^{+}\Big(\frac{T_1}{n}\Big)} \geq \delta\bigg) 
  &\leq  \gamma'+ \sum_{n=\lceil\gamma  \rceil}^{T} \mathbb{P}\bigg(\hat{\mu}_{n}\geq \delta\bigg( 1-\sqrt{\frac{\gamma'}{n}}\bigg)  \bigg) \notag\\
   &\leq    \gamma' + \sum_{n=\lceil\gamma'  \rceil}^{\infty} \exp\bigg( -\frac{\delta^2( \sqrt{n}-\sqrt{\gamma'})^2}{2}\bigg) \label{eq:sum_bound_stat1_hoeffding}\\
   &\leq  \gamma'+1+ \int_{\gamma'}^{\infty} \exp\bigg( -\frac{\delta^2( \sqrt{x}-\sqrt{\gamma'})^2}{2}\bigg) \dd x \notag\\
   &\leq  \gamma'+ 1+\frac{2}{\delta}\int_{0}^{\infty} \Big(\frac{y}{\delta}+\sqrt{\gamma'}\Big)\exp(-y^2/2) \dd y \notag\\
   &\leq  \gamma'+1 +\frac{2}{\delta^2}+\frac{\sqrt{2\pi \gamma'}}{\delta}, \label{eq:sum_bound_stat1_int}
\end{align}
where \eqref{eq:sum_bound_stat1_hoeffding} is the result of Lemma \ref{lem:subguassian} and \eqref{eq:sum_bound_stat1_int} is due to the fact that $\int_{0}^{\infty}y\exp(-y^2/2)\dd y=1$ and $\int_{0}^{\infty}\exp(-y^2/2)\dd y=\sqrt{2\pi}/2$. \eqref{eq:sum_bound_stat1_int} immediately implies the claim in the first statement:
\begin{align}
    \sum_{n=1}^{T}\PP\bigg( \hat{\mu}_n + \sqrt{\frac{4}{n}\log^{+}\Big(\frac{T_1}{n}\Big)} \geq \delta\bigg) 
    \leq & \gamma'+\sum_{n=\lceil \gamma' \rceil}^{T} \PP \bigg(\hat{\mu}_n\geq \delta\bigg( 1- \sqrt{\frac{\gamma'}{n}}\bigg) \bigg) \notag\\
     \leq & \gamma'+1+\frac{2}{\delta^2}+\frac{\sqrt{2\pi\gamma'}}{\delta}.
\end{align}
Plugging $\gamma'\leq4\log^{+}(T_1\delta^2)/\delta^2+1/\delta^2$ to above equation, we obtain
\begin{equation}
        \sum_{n=1}^T \mathbb{P}\bigg(\hat{\mu}_{n}+\sqrt{\frac{4}{n}\log^{+}\bigg(\frac{T_1}{n}\bigg)} \geq \delta \bigg)\leq 1+\frac{4\log^{+}({T_1}{\delta^2})}{\delta^2} +\frac{3}{\delta^2}+\frac{\sqrt{8\pi {\log^{+}({T_1}{\delta^2})}}}{\delta^2}.
   \end{equation}
For the second statement, its proof is similar to that of the first one. Let us define the following quantity: 
 \begin{equation}\rho=\frac{2\log({T}{\delta^2}(\log^2({T}{\delta^2})+1))}{\delta^2}.
\end{equation} 
Note that for all $n\geq1/\delta^2$, it holds that
\begin{align}
    \delta \sqrt{\frac{\rho}{n}}&=\sqrt{\frac{2\log({T}{\delta^2}(\log^2({T}{\delta^2})+1))}{n}}\geq \sqrt{\frac{2}{n}\log\bigg(\frac{T}{n}\bigg(\log^2\frac{T}{n}+1\bigg)\bigg)}.
\end{align}
 Using the same argument in \eqref{eq:sum_bound_stat1_int} we can show that
\begin{align*}
 \sum_{n=1}^T \mathbb{P}\Bigg(\hat{\mu}_{n}+\sqrt{\frac{2}{n}\log\bigg(\frac{T}{n}\bigg(\log^2\frac{T}{n}+1\bigg)\bigg)} \geq \delta \Bigg) &\leq  1+\frac{2\log({T}{\delta^2}(\log^2({T}{\delta^2})+1))}{\delta^2} +\frac{3}{\delta^2}\notag\\
 &\qquad+\frac{\sqrt{4\pi {\log({T}{\delta^2}(\log^2({T}{\delta^2})+1))}}}{\delta^2}.
\end{align*}
To prove the last statement, we borrow the idea from  \cite{menard2017minimax} for proving the regret of kl-UCB$^{++}$. Define $f(\delta)=2/\delta^2\log(T\delta^2/4)$. Then we can decompose the event $\{\exists s: s\leq T\}$ into two cases: $\{\exists s: s\leq f(\delta)\}$ and $\{\exists s: f(\delta)\leq s\leq T\}$.
\begin{align}\label{eq:sum_bound_stat2_decomp}
    & \mathbb{P}\bigg( \exists s \leq T: \hat{\mu}_s+\sqrt{\frac{2}{s}\log\Big(\frac{T}{s}\Big(\log^2\frac{T}{s}+1\Big)\Big)}+\delta\leq 0 \bigg) \notag\\
    &\leq  \underbrace{\mathbb{P}\bigg(\exists s \leq f(\delta):  \hat{\mu}_s \leq -\sqrt{\frac{2}{s}\log\Big(\frac{T}{s}\Big(\log^2\frac{T}{s}+1\Big)\Big)}\bigg)}_{A_1} + \underbrace{\mathbb{P}(\exists s, f(\delta)\leq s \leq T: \hat{\mu}_s\leq -\delta)}_{A_2}. 
\end{align}
Note that when $T\delta^2\geq 4e^3$, $f(\delta)\geq 0$.  
Let $\beta>1$ be a parameter that will be chosen later. Applying the peeling technique, we can bound term $A_1$ as follows.
 \begin{equation}\label{eq:sum_bound_stat2_decomp_A1}
     A_1\leq \sum_{\ell=0}^{\infty}\underbrace{\mathbb{P}\bigg(\exists s, \frac{f(\delta)}{\beta^{\ell+1}}\leq s \leq \frac{f(\delta)}{\beta^{\ell}}: \hat{\mu}_s+\sqrt{\frac{2}{s}\log\Big(\frac{T}{s}\Big(\log^2\frac{T}{s}+1\Big)\Big)}\leq 0 \bigg)}_{A_{1}^{\ell}}.
 \end{equation}
For each $\ell=0,1,\ldots$, define $\gamma_l$ to be
\begin{equation}
 \gamma_{\ell}= \frac{\beta^{\ell}}{f(\delta)}\log\bigg(\frac{T\beta^\ell}{2f(\delta)} \bigg( 1+\log^2 \frac{T}{2f(\delta)}\bigg) \bigg),
\end{equation}
which by definition immediately implies
\begin{align*}
    \sqrt{2\gamma_l}=\sqrt{\frac{2\beta^{\ell}}{f(\delta)}\log\bigg(\frac{T\beta^\ell}{2f(\delta)} \bigg( 1+\log^2 \frac{T}{2f(\delta)}\bigg) \bigg)}\leq\sqrt{\frac{2}{s}\log\bigg(\frac{T}{2s} \bigg(\log^2 \frac{T}{s}\bigg)+1 \bigg)},
\end{align*}
where in the above inequality we used the fact that $s\leq f(\delta)/\beta^{\ell}$ and that $f(\delta)\geq s/2$ since $\beta>1$. Therefore, we have
\begin{align} \label{eq:sum_bound_stat2_decomp_A1l}
  & \mathbb{P} \bigg(\exists s, \frac{f(\delta)}{\beta^{\ell+1}}\leq s \leq \frac{f(\delta)}{\beta^{\ell}}: \hat{\mu}_s+\sqrt{\frac{2}{s}\log\Big(\frac{T}{s}\Big(\log^2\frac{T}{s}+1\Big)\Big)}\leq 0 \bigg)  \notag\\
  &\leq  \mathbb{P} \bigg(\exists \frac{f(\delta)}{\beta^{\ell+1}}\leq s \leq \frac{f(\delta)}{\beta^{\ell}}: \hat{\mu}_s+\sqrt{2\gamma_\ell}\leq 0 \bigg) \notag\\
  &\leq  \exp\bigg( -\frac{f(\delta)}{\beta^{\ell+1}}\gamma_\ell\bigg)\notag\\
  &= e^{-\ell \log(\beta)/\beta-C/\beta},
\end{align}
where the second inequality is by Doob's maximal inequality (Lemma \ref{lemma:maximal_ineq}), the last equation is due to the definition of $\gamma_{\ell}$, and the parameter $C$ is defined to be
 \begin{equation}
     C := \log \bigg(    \frac{T}{2f(\delta)} \bigg(1+\log^2 \frac{T}{2f(\delta)} \bigg)\bigg).
 \end{equation}
Substituting \eqref{eq:sum_bound_stat2_decomp_A1l} back into \eqref{eq:sum_bound_stat2_decomp_A1}, we get 
\begin{align*}
    A_1\leq \sum_{\ell=0}^{\infty} e^{-\ell \log(\beta)/\beta-C/\beta}& =\frac{e^{-C/\beta} }{1-e^{-\log(\beta)/\beta}}\leq \frac{e^{1-C/\beta}}{e^{\log(\beta)/\beta}-1} \leq 2e \max (\beta,\beta/(\beta-1)) e^{-C/\beta},
\end{align*}
where the second inequality is due to $\log \beta\leq\beta$ and thus $e^{\log(\beta)/\beta}\leq e$, and the last inequality comes from Lemma~\ref{lem:ebeta}.
Since $T\delta^2 \geq 4e^3$, we have $T/(2f(\delta))=T\delta^2/(4\log(T\delta^2/4))\geq \sqrt{T\delta^2/4}\geq e^{3/2}$, which further implies
\begin{equation}
    C=\log \bigg(    \frac{T}{2f(\delta)} \bigg(1+\log^2 \frac{T}{2f(\delta)} \bigg)\bigg) \geq \log\bigg(\frac{T}{2f(\delta)}\bigg)=\log \bigg( \frac{T\delta^2}{4\log (\frac{T\delta^2}{4})}\bigg) \geq 3/2.
\end{equation}
Now we choose $\beta:=C/(C-1)$, so that $1<\beta\leq 2C$ and $\beta/(\beta-1)=C$. Together with the definition of $f$, this choice immediately yields 
\begin{align}
\label{eq:lemc.1-G.19}
    A_1\leq 4eCe^{-C/\beta}=4e^2Ce^{-C}.
\end{align}
 Note that
\begin{align}
\label{eq:alg-mini-added}
    Ce^{-C}&=\bigg( \frac{T}{2f(\delta)}\bigg( 1+\log^2  \frac{T}{2f(\delta)}  \bigg) \bigg)^{-1}\log\bigg( \frac{T}{2f(\delta)}\bigg( 1+\log^2  \frac{T}{2f(\delta)}  \bigg) \bigg)\notag\\
    &\leq   \frac{2f(\delta)}{T\log^2(T/(2f(\delta)))} \log\bigg( \frac{T}{2f(\delta)}\bigg( 1+\log^2  \frac{T}{2f(\delta)}  \bigg) \bigg)\notag\\
    &\leq \frac{4f(\delta)}{T\log(T/(2f(\delta)))}\notag\\
    &=\frac{8\log(T\delta^2/4)}{T\delta^2\log([T\delta^2/4]/\log(T\delta^2/4))}\notag\\
    &\leq \frac{16}{T\delta^2},
\end{align}
where in the second and the third inequalities, we used the fact that that for all $x\geq e^{3/2}$, \begin{equation}
    \frac{\log (x(1+\log^2x) )}{\log x} \leq 2 \ \ \ \ \ \ \ \ \text{and}   \ \ \ \ \ \ \ \ \ \ \frac{\log x}{\log(x/
    \log x)} \leq 2.
\end{equation}
Therefore, we have proved so far $A_1\leq 64e^2/(T\delta^2)$. For term $A_2$ in \eqref{eq:sum_bound_stat2_decomp}, we can again apply the maximal inequality in Lemma \ref{lemma:maximal_ineq} and obtain
\begin{equation}
    A_2=\mathbb{P}(\exists s, f(\delta)\leq s \leq T: \hat{\mu}_s\leq -\delta)\leq e^{-\delta^2f(\delta)/2}=\frac{4}{T\delta^2}.
\end{equation}
Finally, combining the above results, we get 
\begin{equation}
\label{eq:final-colt17}
    \mathbb{P}\bigg( \exists s \leq f(\delta), \hat{\mu}_s+\sqrt{\frac{2}{s}\log\Big(\frac{T}{s}\Big(\log^2\frac{T}{s}+1\Big)\Big)}+\delta\leq 0 \bigg) \leq \frac{4(16e^2+1)}{T\delta^2}.
\end{equation}
This completes the proof.
\end{proof}

\subsection{Proof of Lemma~\ref{lem:colt17-1}}
\begin{proof}
Recall $\delta\in(0,2/\log^4 T)$. Note that for $s\leq \log^2 T$,
\begin{align*}
  \sqrt{\frac{2}{s}\log\bigg(\frac{eT}{s}\bigg(\log^2\frac{T}{s}+1\bigg)\bigg)}-\delta&\geq  \sqrt{\frac{2}{s}\bigg(1+\log\bigg(\frac{T}{s}\bigg(\log^2\frac{T}{s}+1\bigg)\bigg)\bigg)}-\frac{2}{\log^4 T}  \notag \\
  & \geq  \sqrt{\frac{2}{s}\log\bigg(\frac{T}{s}\bigg(\log^2\frac{T}{s}+1\bigg)\bigg)}.
\end{align*}
Let $a(s)=2/s$, $b(s)=2/s\log(T/s(\log^2 (T/s)+1)$.  The last inequality  is equals to $\sqrt{a(s)+b(s)}-\sqrt{b(s)}\geq 2/\log^4 T$ for $s\leq \log^2 T$, which holds because (i): 
\begin{align*}
    \sqrt{a(s)+b(s)}-\sqrt{b(s)}=a(s)/(\sqrt{a(s)+b(s)}+\sqrt{b(s)});
\end{align*} 
(ii): for $s\leq \log^2 T$, then $a(s)\geq a(\log T)= 2/\log^2 T$,
\begin{align*}
   \sqrt{a(s)+b(s)}+\sqrt{b(s)}\leq \sqrt{a(1)+b(1)}+\sqrt{b(1)}<\log^2 T,
\end{align*}
 hence $a(s)/(\sqrt{a(s)+b(s)}+\sqrt{b(s)})\geq 2/\log^4 T$. 
 Now, we only need to prove 
\begin{align}
 \label{lem:main-equation-33}
      \mathbb{P}\bigg(\exists s\leq \log^2 T: \hat{\mu}_{s}+\sqrt{\frac{2}{s}\log\bigg(\frac{T}{s}\bigg(\log^2\frac{T}{s}+1\bigg)\bigg)}\leq 0 \bigg)\leq \frac{16e^2\log T}{T}.
\end{align}
The rest proof of Lemma~\ref{lem:colt17-1} is similar to the proof of Lemma~\ref{lem:colt17}.
Let 
$A_1$ be the r.h.s. \eqref{lem:main-equation-33} and $f=\log^2 T$.
Then applying the peeling technique, we can bound $A_1$ as follows.
\begin{align*}
    A_1\leq \sum_{\ell=0}^{\infty} \PP \bigg(\exists s, \frac{f}{\beta^{\ell+1}}\leq s \leq \frac{f}{\beta^{\ell}}: \hat{\mu}_{s}+\sqrt{\frac{2}{s}\log\bigg(\frac{T}{s}\bigg(\log^2\frac{T}{s}+1\bigg)\bigg)}\leq 0 \bigg).
\end{align*}
  Similar to \eqref{eq:lemc.1-G.19}, we have $A_1\leq 2e\max (\beta,\beta/(\beta-1))e^{-C/\beta}\leq 4e^2Ce^{-C}$. Then  \eqref{eq:alg-mini-added} becomes
\begin{align}
    Ce^{-C}&= \bigg( \frac{T}{2f}\bigg( 1+\log^2  \frac{T}{2f}  \bigg) \bigg)^{-1}\log\bigg( \frac{T}{2f}\bigg( 1+\log^2  \frac{T}{2f}  \bigg) \bigg)\notag\\
    &\leq   \frac{4f}{T\log(T/(2f))} \notag\\
    &\leq \frac{4\log T}{T},
\end{align}
where the last inequality is due to $f=\log^2 T$. 
Therefore, we have
\begin{align*}
    \mathbb{P}\bigg(\exists s \leq \log^2 T:  \hat{\mu}_s +\sqrt{\frac{2}{s}\log\Big(\frac{eT}{s}\Big(\log^2\frac{T}{s}+1\Big)\Big)}-\delta\leq 0\bigg) \leq \frac{16e^2\log T}{T} .
\end{align*}
This completes the proof.
\end{proof}

\bibliography{bandits}
\bibliographystyle{ims}
\end{document}